\documentclass[11pt]{article}
\usepackage[preprint]{acl}
\usepackage{times}
\usepackage{latexsym}
\usepackage[T1]{fontenc}
\usepackage[utf8]{inputenc}
\usepackage{microtype}
\usepackage{inconsolata}
\usepackage{url}
\usepackage{graphicx}
\usepackage{booktabs}
\usepackage{amsmath}
\usepackage{amsthm}
\usepackage{enumitem}
\newtheorem{proposition}{Proposition}

\title{One Score, Two Decisions:\\ Selective Prediction on the Rare-Disease Tail}
\author{%
  \bfseries
  Zhaoyang Jiang\textsuperscript{1} \quad
  Zhizhong Fu\textsuperscript{3} \quad
  Yunsoo Kim\textsuperscript{1} \quad
  Zicheng Li\textsuperscript{4} \\[0.3em]
  \bfseries
  Xuanqi Peng\textsuperscript{1} \quad
  Fei Teng\textsuperscript{1} \quad
  Jiacong Mi\textsuperscript{2} \quad
  Honghan Wu\textsuperscript{1}\thanks{Corresponding author.} \\[0.5em]
  \normalfont
  \textsuperscript{1}School of Health \& Wellbeing, University of Glasgow, Glasgow, UK \\
  \textsuperscript{2}Department of Respiratory and Critical Care Medicine, Shanghai Sixth People's Hospital, \\
  Shanghai Jiao Tong University School of Medicine, Shanghai, China \\
  \textsuperscript{3}School of Life Science and Technology, \\
  University of Electronic Science and Technology of China, Chengdu, China \\
  3167645J@student.gla.ac.uk,\quad
  Honghan.Wu@glasgow.ac.uk,
}
\begin{document}
\maketitle

\begin{abstract}
Given a patient's clinical findings, a diagnostic system ranks possible diseases and must decide when to
endorse its first prediction or defer it for review. This decision is usually made by thresholding the top
score. Selective prediction over ranked outputs begins with two checks. First, the ranker must produce enough
correct top-ranked predictions to make the target feasible. Across $2{,}000$ patient records stratified by
disease prevalence, eight small open-weight LLMs achieve at most $4.6\%$ Recall@1 on ultra-rare diseases. At
$10\%$ coverage, even a perfect confidence ranking of their existing predictions therefore cannot reach
$50\%$ selective accuracy. More accurate models pass the same check, showing that the limit is
regime-specific. Second, the confidence signal must match the decision being made. For fixed-candidate
rankers, the top-two margin cancels components shared across candidates. On phenotype-only Exomiser, it
selects $10\%$ of cases at $29.0\%$ accuracy, compared with $13.3\%$ overall, while the top score provides no
reliable gate. Yet that cancellation can remove information needed to detect whether the candidate list
contains an answer. SciFact retrieval and biomedical entity linking confirm this distinction. Finally, we
prove that unlabelled scores alone cannot determine whether switching to the margin will help.
\end{abstract}

\begin{figure*}[t]
\centering\includegraphics[width=\textwidth]{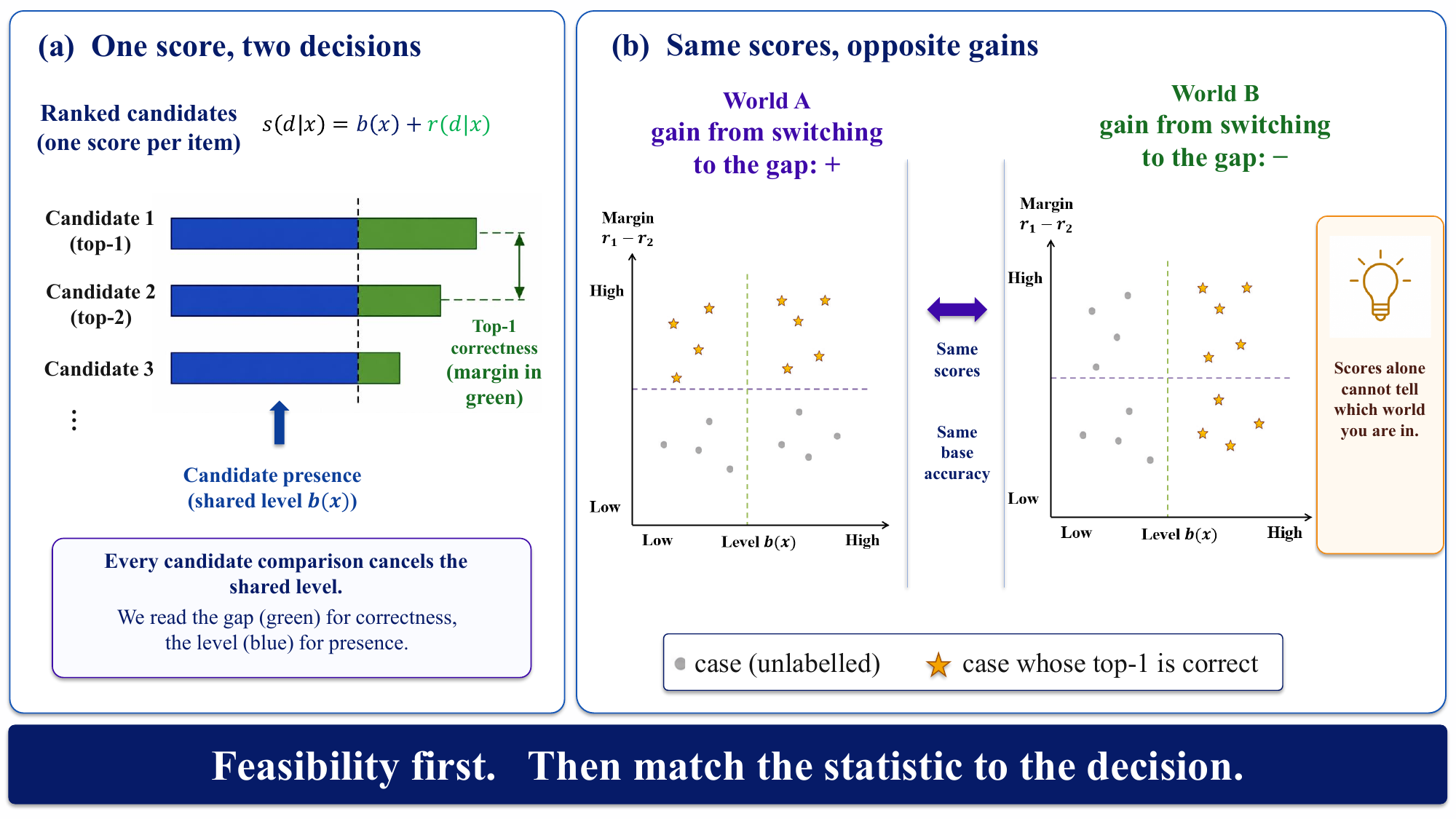}
\caption{Two of the paper's three claims; the first, feasibility, is Eq.~\eqref{eq:ceiling} and
\S\ref{sec:calib}. \textbf{(a)} Eq.~\eqref{eq:scoredecomp} splits each score into a level shared by every
candidate for that case and a candidate-specific residual. We find correctness read from the gap the
residuals leave between the leaders, while presence needs case-level information that the gap discards and
that, by Eq.~\eqref{eq:contrast}, no zero-sum contrast over the candidates can recover
(\S\ref{sec:transfer}). \textbf{(b)} Two
joint laws with the same unlabelled scores and the same base accuracy give gains of opposite sign, so no
functional of those scores identifies which gate is better (Prop.~\ref{prop:unident}).}
\label{fig:overview}
\end{figure*}

\section{Introduction}
\label{sec:intro}
Rare-disease diagnosis is a ranking problem. Given a patient's phenotypes (structured findings such as
seizures or short stature), a diagnostic system must order thousands of possible diseases. Large language
models (LLMs) can turn that short input into a broad differential diagnosis
\citep{singhal2023clinical,kanjee2023accuracy,chen2024rarebench}, which makes them attractive on the
ultra-rare tail, where diagnostic support is needed most.

A differential only helps if the clinician knows when to trust the system's first suggestion. Selective
prediction handles this by endorsing the top prediction on high-confidence cases and deferring the rest for
review \citep{chow1970optimum,elyaniv2010foundations,geifman2017selective}, which in our setting withholds endorsement of the first candidate without hiding the
remaining differential: a confident error leads to unnecessary testing and counselling, while excessive abstention
makes the system unhelpful. The same top score is
routinely thresholded for two different decisions, whether the top candidate is correct and whether any
correct candidate is present \citep{hendrycks2017baseline,sevgili2022neuralel}. We show that the two need
different checks and different signals.

The first check is feasibility. If a predictor is correct on a fraction $p$ of cases and the system answers a
fraction $c$, even a perfect confidence ranking cannot reach selective accuracy above $\min(1,p/c)$. A target above that ceiling cannot be reached by
recalibrating or rescoring the same predictions, because confidence can reorder correct answers but cannot
create them. The check matters on distributional tails, where pooled
accuracy hides a much lower tail accuracy \citep{ransohoff1978spectrum}.

For predictors that pass, the signal depends on the decision. A top score can carry a case-level component
shared by every candidate, and the gap to the runner-up removes it. That helps in deciding whether the leader
is correct, since the gap measures its separation from the nearest alternative. But the shared level is what indicates whether the candidate
set fits the input at all, as when a retriever holds no relevant document or a mention has no entry in the
knowledge base, so one subtraction can help correctness and hurt presence. Nor does any other comparison
among the candidates recover it: every zero-sum contrast cancels the level exactly, so a system that needs it
must look outside the contest. Figure~\ref{fig:overview} summarises the distinction and
\S\ref{sec:framework} states it formally.

We test this on a uniform $2{,}000$-case sample of $10{,}374$ patient records, read by the Orphanet prevalence
class of each diagnosis. Across eight small open-weight LLMs ultra-rare Recall@1 is at most $4.6\%$, so at
$10\%$ coverage even perfect confidence is capped at $46\%$. Two of five medical-specialised models clear a $50\%$
target at the point estimate and across their whole interval, and the three frontier configurations reach
$15.8$--$22.4\%$, making the ceiling non-binding, so the limit is specific to low-accuracy predictors rather
than to confidence.

In that regime phenotype rankers, which compare a patient's findings with curated disease profiles, are the
stronger starting point under our leave-source-out control. On phenotype-only Exomiser the top score yields
no usable operating point while the gap selects $10\%$ of ultra-rare cases at $29.0\%$ against a $13.3\%$ base
rate. Subtracting the shared component from a retriever we can recompute lifts its top-score gate, with a
post-hoc temperature, to the level of its gap, testing the mechanism rather than inferring it. The gap also
beats the raw score for correctness on SciFact retrieval and on entity linking with an unnormalized scorer.
Masking relevant candidates then separates the two: on SciFact the gap becomes the worst of three presence
signals, and on entity linking it is never the best. Finally, we prove that
unlabelled scores, even at fixed base accuracy, cannot determine the gain from switching.

Code is available at
\url{https://github.com/Anonymous-Awesome-Submissions/one_score_two_decisions}.

\section{Related Work}
\paragraph{Selective prediction and deferral.} Calibration and uncertainty estimation sharpen the confidence
attached to a \emph{fixed} set of top-1 predictions
\citep{xin2021abstention,kamath2020selective,wen2025abstention,guo2017calibration,ovadia2019trust,farquhar2024semantic,kuhn2023semantic,angelopoulos2023conformal,varshney2022selective}:
they select among existing answers without changing them, so what they reach is bounded by
Eq.~\eqref{eq:ceiling}. Deferral, cascades and adaptive retrieval work in the other regime, routing a case to
a second predictor and changing the base accuracy itself
\citep{mozannar2020defer,jitkrittum2023cascade,jeong2024adaptiverag,ni2024overconfidence,feng2024abstain}. We
read the oracle ceiling \citep{rabanser2023private} as the test that tells the regimes apart.

\paragraph{Confidence for ranked outputs.} The top score, the gap between the leaders and trust scores over
local geometry are established signals
\citep{jiang2018trust,scheffer2001margin,dalitz2009knn,lowe2004sift,liang2024selective}, as is subtracting a
per-input nuisance: by query length, an impostor cohort, the number of query terms in phenotype-driven
diagnosis, or a normalization that makes scores comparable across queries
\citep{zhou2007wig,shtok2012nqc,auckenthaler2000tnorm,schulz2011exact,karakos2020detection,bogolin2022qbnorm}.
All of it removes the shared level \emph{to rank better}; our question is which decision still needs it, which
Eq.~\eqref{eq:scoredecomp} makes testable. Concurrently \citet{wagner2026twoaxes} separate correctness from
answerability with two internal LLM signals, and \citet{wang2025targ} gate retrieval on a top-two margin under
a monotone link \S\ref{sec:ranklimit} shows cannot be chosen from unlabelled scores; we study the split in the
scores a ranker already emits.

\paragraph{Long-tail knowledge.} That literature defines rarity by corpus frequency
\citep{kandpal2023longtail,razeghi2022impact,mallen2023trust,ni2025popularity,sun2024headtotail}; we stratify
by Orphanet prevalence. Holding corpus frequency fixed leaves the per-disease gap essentially intact, on too few matched pairs to
remove the exposure explanation rather than bound it (App.~\ref{app:macro}).

\paragraph{Rare-disease diagnosis.} Phenotype prioritizers rank diseases against curated profiles
\citep{kohler2009phenomizer,smedley2015exomiser,robinson2020lirical,jagadeesh2019phrank}; general-purpose LLMs
trail them \citep{reese2026benchmarking}, and specialised or agentic systems add training, retrieval and tools
without prevalence-stratified or decontaminated evaluation \citep{yang2025rareseek,zhao2025deeprare}. A
meta-analysis calls for the prevalence-stratified design we adopt \citep{nguyen2026metaanalysis}.
\citet{elmofty2026retrieval} find the crossover turning on retrieval \emph{coverage}, whether the answer
enters the pool at all; our $p$ asks whether the predictor already places it first, which is what decides the
feasible operating points (App.~\ref{app:margin}).

\section{Selective Prediction over Ranked Outputs}
\label{sec:framework}
Selective prediction is usually posed as a choice of confidence score. For ranked outputs that question
comes too early: one must first ask whether the predictor has produced enough correct top-ranked answers to
reach the operating point, and then what the gate is meant to certify, that the leading candidate is correct
or that the candidate set contains a valid answer at all. Three steps follow: a feasibility test, a
decision-specific decomposition of the scores, and a limit on what unlabelled scores can settle.

\subsection{Feasibility Before Confidence Estimation}
\label{sec:feasibility}
Consider $N$ cases. Let $y_i\!=\!1$ if the predictor's top-ranked answer is correct on case $i$ and $0$
otherwise, so its base accuracy is $p=\frac{1}{N}\sum_i y_i$. A confidence rule $q$ selects an answered set
$A_q$ with realized coverage $c=|A_q|/N$, including any effect of tied scores, and selective accuracy
$\operatorname{Acc}_{\mathrm{sel}}(q,c)=\frac{1}{|A_q|}\sum_{i\in A_q}y_i$. The answered set cannot contain more
correct predictions than exist in the whole sample, nor more than it has cases,
\[
\sum_{i\in A_q} y_i\;\leq\;\min\Big(|A_q|,\ \textstyle\sum_{i=1}^{N} y_i\Big),
\]
which on dividing by $|A_q|$ gives
\begin{equation}
\operatorname{Acc}_{\mathrm{sel}}(q,c)\;\le\;\min\!\left(1,\frac{p}{c}\right).
\label{eq:ceiling}
\end{equation}
This is the standard oracle ceiling behind selective prediction
\citep{chow1970optimum,elyaniv2010foundations,geifman2017selective,rabanser2023private}, used here as a
feasibility test rather than claimed as a new bound. If deployment requires selective accuracy $\tau$ at
coverage $c$, Eq.~\eqref{eq:ceiling} implies the necessary condition
\begin{equation}
p \;\geq\; \tau c .
\label{eq:feasibility}
\end{equation}
When it fails, no recalibration, uncertainty estimator or confidence score that leaves the predictor's
answers unchanged can reach the target, because the required correct answers do not exist to be selected. The
remedies are all of a different kind: a more accurate predictor, a larger returned candidate set, or a target
the deployment can afford to relax. When the condition holds the target is merely possible, since the gate
must still place the correct cases first. For a top-$k$ decision the same argument applies with $y_i$ redefined as whether the
gold answer appears in the top $k$, so $p$ becomes Recall@$k$. We read the condition across a range of $\tau$
\citep{whitehead2022reliable}, labelling it at $\tau{=}50\%$ where a single point helps
\citep{vickers2006decision}.

\subsection{Correctness and Candidate Presence}
\label{sec:signals}
Now let one scorer evaluate a fixed candidate set $D$. For a case $x$ let $d_1(x)=\arg\max_{d\in D}s(d\mid
x)$ be the leading candidate and $d_2(x)$ the runner-up, with scores $s_1(x)$ and $s_2(x)$. Writing
$\operatorname{rel}(d,x)\!=\!1$ when candidate $d$ is correct for $x$, two labels arise:
\[
\begin{aligned}
y_{\mathrm{corr}}(x)&=\operatorname{rel}\big(d_1(x),x\big),\\[-1pt]
y_{\mathrm{pres}}(x)&=\max\nolimits_{d\in D}\operatorname{rel}(d,x).
\end{aligned}
\]
The first asks whether the leader should be endorsed, the second whether the candidate set is usable at all.
Correctness implies presence but not conversely: a retriever can hold a relevant document without ranking it
first, and a linker can rank confidently when the right entity is absent from its knowledge base. The $p$ of
\S\ref{sec:feasibility} is the base rate of $y_{\mathrm{corr}}$. Both decisions are usually gated by the same
number, the top score $q_{\mathrm{top}}(x)=s_1(x)$, and that becomes consequential when scores carry a
component shared by every candidate for the same case. Suppose
\begin{equation}
s(d\mid x)=b(x)+r(d\mid x),
\label{eq:scoredecomp}
\end{equation}
with $b(x)$ a case-level offset and $r(d\mid x)$ what distinguishes candidates. The top score retains both,
$q_{\mathrm{top}}(x)=b(x)+r_1(x)$, whereas the margin $q_{\mathrm{margin}}(x)=s_1(x)-s_2(x)$ removes the
shared part:
\begin{equation}
q_{\mathrm{margin}}(x)=r_1(x)-r_2(x).
\label{eq:margincancel}
\end{equation}
Nothing here is special to the top two. For any weights $w$ over $D$ summing to zero,
\begin{equation}
\sum_{d\in D} w_d\, s(d\mid x)\;=\;\sum_{d\in D} w_d\, r(d\mid x),
\label{eq:contrast}
\end{equation}
so \emph{every} zero-sum contrast among one case's candidates cancels $b(x)$ identically. The weights may
depend on the candidates' rank order, which $b(x)$ does not change, so this covers the margin,
$w=(1,-1,0,\dots,0)$, and equally the statistic one would reach for to \emph{estimate} the case level, the
top score against the mean of the field, $w=(1,-\tfrac{1}{K-1},\dots,-\tfrac{1}{K-1})$. The top score of course still carries $b(x)$, but
carries it added to $r_1(x)$, and no comparison among the candidates can separate the two. Isolating a level
that is comparable across cases therefore needs something the contest does not contain: a scale already
fixed across cases, as a bounded cosine has, or a case-side estimate of the offset, for which
\S\ref{sec:transfer} uses a query's own IDF mass.

Splitting $b(x)=\nu(x)+e(x)$ turns this into a decision-specific hypothesis rather than a universal
ordering: a nuisance $\nu$, moving for reasons
unrelated to whether any candidate is correct, and evidence $e$ about whether one is. A contrast discards
both, which should help $y_{\mathrm{corr}}$ and hurt $y_{\mathrm{pres}}$; a case-side correction discards
$\nu$ and keeps $e$, so it should help both; the raw score keeps both, and can serve $y_{\mathrm{pres}}$ only
where $\nu$ is small. Candidate-dependent noise survives every subtraction. Which regime a given scorer is in
is an empirical question rather than a consequence of the algebra, and \S\ref{sec:scope} settles it for one
system by removing $b(x)$ from it.

The same point has a probabilistic reading. If the scores are unnormalized log-potentials, so that
$P_T(d\mid x)\propto\exp(s(d\mid x)/T)$, then
\begin{equation}
s_1(x)-s_2(x)=T\log\frac{P_T(d_1\mid x)}{P_T(d_2\mid x)} ,
\label{eq:marginodds}
\end{equation}
the log-odds of the leader against its nearest competitor up to temperature: the normalizer and every
candidate-independent offset cancel, as in conditional logit and partial likelihood
\citep{mcfadden1974conditional,cox1972regression,liang2024selective,hengsoh2025rlog}. Under $s(d\mid
x)\mapsto\alpha s(d\mid x)+\beta(x)$ with $\alpha>0$ the top score moves with the arbitrary case-level
$\beta(x)$ while the margin keeps its ordering of cases, and so its risk--coverage curve, unchanged
(App.~\ref{app:margin_invariance}).

\subsection{Limits of Label-Free Gate Selection}
\label{sec:ranklimit}
Write
$G(c)=\operatorname{Acc}_{\mathrm{sel}}(q_{\mathrm{margin}},c)-\operatorname{Acc}_{\mathrm{sel}}(q_{\mathrm{top}},c)$
for the gain at coverage $c$. One would like to estimate $G$ before switching, from the scores alone. That is
not a matter of finding the right diagnostic.

\begin{proposition}[Unidentifiability from unlabelled scores]
\label{prop:unident}
Let $P_S$ be the law of the candidate scores on unlabelled inputs and $T$ any functional of $P_S$. Two joint
laws can agree on $P_S$ and on base accuracy while $G(\tfrac12)$ takes opposite signs; no such $T$
identifies $G$.
\end{proposition}

\begin{proof}
Let $U,V$ be independent and uniform on the unit interval and put $(s_1,s_2)=(U,\,U-V)$, so $s_1\!\geq\!s_2$, the
margin is $m=V$, and $m$ is independent of $s_1$. Both laws use these scores, so $P_S$ and $T(P_S)$ agree. Let
the gold candidate be $d_1$ exactly when $V>\tfrac12$ under $P_{+}$ and exactly when $U>\tfrac12$ under
$P_{-}$; base accuracy is $\tfrac12$ in both. At $c=\tfrac12$ the margin selects with accuracy $1$ under
$P_{+}$ and the top score with $\tfrac12$, so $G=+\tfrac12$; under $P_{-}$ the two exchange and
$G=-\tfrac12$.
\end{proof}

Rank-only statistics are the special case, since candidate order is a functional of $P_S$, as a monotone map
shows directly: a strictly increasing $f$ fixes every ranking while $f(s_1)-f(s_2)$ reorders cases
\citep{stevens1946scales,wang2019miscalibration}. The quantifier is the unlabelled score distribution and no
wider. A method that reads the inputs, the candidate texts or the scorer itself, or that intervenes on scores
rather than observing them, is outside the statement, which is why \S\ref{sec:scope}'s mechanism
experiment, an intervention, is not forbidden by it. What is ruled out is reading the answer off the scores a
deployed system already returns.

The margin carries a domain condition of its own: it is meaningful when one predictor scores a common
candidate set and the top two are distinct alternatives on a common scale. Repeated samples from a free-form
generator are not that: the two highest-scoring samples are usually the identical string, leaving no
runner-up at all (App.~\ref{app:margin_scope}).

The two results order the questions. Eq.~\eqref{eq:feasibility} asks whether a predictor can reach the required
operating point at all, and is settled before any confidence estimator is compared.
Eq.~\eqref{eq:scoredecomp} then motivates which part each decision should read, a structural hypothesis we
test rather than assume, and Prop.~\ref{prop:unident} withholds the size of the gain.

\section{Data and Experimental Setup}
\label{sec:bench}
Our benchmark is Phenopacket Store v0.1.27 \citep{jacobsen2022phenopacket,danis2024phenopacketstore}, which
collects $10{,}374$ real patient cases, each a set of observed Human Phenotype Ontology terms
\citep{kohler2021hpo} with the patient's confirmed OMIM diagnosis. The $780$ diagnoses that appear are the
evaluation labels; the retriever ranks all $8{,}553$ OMIM entries with an HPOA profile, the fixed jointly
scored set the margin needs. App.~\ref{app:terms} glosses the clinical-genetics terms used below.

To stratify by rarity we map each gold OMIM entry to its Orphanet prevalence class \citep{orphadata}, taking
the rarest when an entry carries several. That gives $4{,}780$ ultra-rare cases (below $1/10^6$, or
$1$--$9/10^6$) and $1{,}167$ less-rare ones (at least $1/10^5$); a third group whose prevalence Orphanet does
not document is analysed separately, only $372$ of the $780$ diseases carrying a documented class. The
collapse survives a point-prevalence rule and a finer four-band split (App.~\ref{app:strat}). Every model sees
the present HPO terms and returns a ranked top-5 differential with a verbalized confidence per entry, as free
text rather than a selection from a list.

We test four classes of system, named individually in Table~\ref{tab:collapse}. The first is the deployable
regime the collapse claim is about: eight small open-weight models a hospital can run on-premises, from the
Qwen \citep{qwen25vl,qwen25}, Llama \citep{grattafiori2024llama3}, Mistral \citep{labrak2024biomistral},
InternLM and Yi families. The second is medical specialisation: five models
\citep{ankit2024openbiollm,chen2024huatuogpto1,sellergren2025medgemma,baichuan2025m2} at the same prompt,
decoding and linker, two of them controlled contrasts, OpenBioLLM-8B against our Llama-3.1-8B row at fixed
base model and OpenBioLLM-70B against its own $8$B sibling at fixed recipe (App.~\ref{app:tabnotes}). The
collapse claim is stated for the eight-model baseline group and not for these. The third is a frontier generality check, DeepSeek-V4-Flash/Pro
\citep{deepseekv4} at $284$B/$1.6$T parameters ($13$B/$49$B active), open-weight but API-served and so
off-regime, with reasoning disabled to isolate scale and enabled to add test-time compute. The fourth comprises two phenotype rankers: our own information-content overlap
retriever, which scores all $8{,}553$ profiled OMIM entries jointly, and Exomiser
\citep{smedley2015exomiser} in phenotype-only mode on the same cases, a tool clinicians run, which ranks
genes rather than diseases (App.~\ref{app:exomiser}). Open models are decoded greedily with vLLM \citep{kwon2023vllm}; compute, versions and seeds are in
App.~\ref{app:repro}.

The rare-disease benchmark always contains the gold diagnosis, so it can ask whether the top-ranked answer
is correct but not whether any correct candidate is present. Two standard ranking tasks supply the second
decision. On SciFact \citep{wadden2020scifact} the released ColBERTv2 checkpoint
\citep{santhanam2022colbertv2} ranks $5{,}183$ abstracts for $300$ claims; on BC5CDR \citep{li2016bc5cdr}
and MedMentions \citep{mohan2019medmentions} SapBERT and BM25 \citep{robertson2009bm25} rank ontology
entries for a mention. Masking a random slice of the corpus or ontology then leaves some inputs with no valid
candidate \citep{zhu2023nil}, so one run scores both decisions off the same scores; rates and preprocessing
are in App.~\ref{app:transfer}.

Generated disease names are linked to OMIM with SapBERT \citep{liu2021sapbert}; the phenotype rankers return
identifiers and use no linker, so the tail comparison does not rest on one. We report Recall@1/@5; selective accuracy as a function of coverage;
AUROC for top-1 correctness and, on the transfer tasks, for candidate presence; and ECE per bin for the
verbalized confidence signals. App.~\ref{app:linkaudit} audits the linker on held-out surface forms.

\section{Results}
\emph{Convention.} Numbers below are per-case (micro); per-disease averaging shrinks the collapse to
$\sim\!2$--$6\times$ and preserves the tail contrast, established once in App.~\ref{app:macro}.

\subsection{Performance across Prevalence Strata}
\label{sec:collapse}
Table~\ref{tab:collapse} shows the collapse across eight small open models spanning five families:
Recall@1 falls from a $\sim\!40\%$ less-rare anchor to $\leq\!5\%$ on the ultra-rare tail, and it is not
Qwen-specific, the two sharpest ratios being Mistral's and InternLM's. Five medical-specialised models split $2$ over the feasibility bar of \S\ref{sec:calib} and $3$ under it,
the best reaching $7.6\%$, and the one controlled pair we have, OpenBioLLM-8B against our Llama-3.1-8B row,
is $-1.0$pp (App.~\ref{app:tabnotes}).

Scale and reasoning move the tail substantially without closing it: DeepSeek-V4-Pro reaches $18.6\%$ and
reasoning $22.4\%$, enough to lift the feasibility ceiling (\S\ref{sec:calib}) though still below the
retriever's $25.6\%$ (\S\ref{sec:reframe}) and not resolvably so. A paired disease-clustered test separates the retriever
from $7$ of the $10$ configurations we can recompute, and the three it does not are exactly the frontier ones,
trailing by $3.1$ to $9.8$pp with intervals containing zero. Against the small models the gap is decisive and in triage
(\S\ref{sec:calib}). The reasoning columns also need a common denominator: on the $849$ cases both configurations returned the
reasoning gain is $+1.9$pp, half of what the printed columns imply (App.~\ref{app:tabnotes}).

\begin{table*}[t]
\centering\footnotesize
\begin{tabular}{@{}llccc@{}}
\toprule
 & & \multicolumn{2}{c}{Recall@1 (\%)} & \\
\cmidrule(lr){3-4}
Model & Family & Less-rare & Ultra-rare & Less/ultra ratio \\
\midrule
\multicolumn{5}{@{}l}{\emph{small open-weight models}} \\
Qwen2.5-VL-7B  & Qwen     & 36.5 & 1.7 & $21\times$ \\
Qwen2.5-14B    & Qwen     & 40.5 & 4.2 & $9.8\times$ \\
Qwen2.5-VL-32B & Qwen     & 40.6 & 4.3 & $9.4\times$ \\
Llama-3.1-8B   & Llama    & 39.2 & 4.6 & $8.6\times$ \\
Mistral-7B-v0.3 & Mistral & 40.1 & 1.6 & $25\times$ \\
BioMistral-7B  & Mistral  & 36.0 & 2.4 & $15\times$ \\
InternLM2.5-7B & InternLM & 32.9 & 1.4 & $24\times$ \\
Yi-1.5-9B      & Yi       & \phantom{0}7.2 & 0.7 & $9.7\times$ \\
\midrule
\multicolumn{5}{@{}l}{\emph{medical-specialised models (on-premises, $8$B--$70$B)}} \\
OpenBioLLM-8B   & Llama   & 36.5 & 3.6 & $10\times$ \\
HuatuoGPT-o1-8B & Llama   & 35.6 & 4.2 & $8\times$ \\
MedGemma-27B    & Gemma   & 37.4 & 3.7 & $10\times$ \\
Baichuan-M2-32B & Baichuan & 46.9 & 7.6 & $6\times$ \\
OpenBioLLM-70B  & Llama    & 40.5 & 6.7 & $6\times$ \\
\midrule
\multicolumn{5}{@{}l}{\emph{frontier MoE models (API)}} \\
DeepSeek-V4-Flash          & DeepSeek & 40.5 & 15.8 & $2.6\times$ \\
DeepSeek-V4-Pro (reason.\ off) & DeepSeek & 48.2 & 18.6 & $2.6\times$ \\
DeepSeek-V4-Pro (reason.\ on)  & DeepSeek & 49.5 & 22.4 & $2.2\times$ \\
\midrule
\multicolumn{5}{@{}l}{\emph{phenotype prioritizers}} \\
IC-overlap retriever (ours) & -- & 31.4 & 25.6 & $1.2\times$ \\
Exomiser (phenotype-only) & -- & 21.2 & 13.3 & $1.6\times$ \\
\bottomrule
\end{tabular}
\caption{Per-case Recall@1 (\%) by disease-prevalence stratum, every system on the same HPO input. Every
generative row and Exomiser are scored on one uniform $2{,}000$-case sample ($939$ ultra-rare); the retriever
needs no generation and is scored on all $10{,}345$ eligible cases, returning the same $25.6\%$ tail value on
the sample (App.~\ref{app:robust}). Two rows have their own denominator: reasoning-on returned $849$ of the
$939$ ultra-rare cases and Qwen2.5-VL-32B $901$, and the dropped cases are the harder ones. Retriever scores
use the leave-source-out control; Exomiser ranks genes and the rest diseases (App.~\ref{app:tabnotes}).}
\label{tab:collapse}
\end{table*}

\subsection{Feasibility Before Calibration}
\label{sec:calib}
Stated confidence does not warn a clinician. It stays badly \emph{over}confident on the tail while its
discrimination runs from near chance to $0.90$ AUROC for Qwen-32B (App.~\ref{app:calib}), which is what
makes it dangerous rather than merely poor.

Good discrimination still does not deliver the specified operating point (Eq.~\eqref{eq:ceiling}), and this
cuts two ways. At $c{=}10\%$ the ceiling is simply $p/c$, which for the eight small models runs from $7\%$ to $46\%$, with
three of the five medical models inside that band and two above it. Any target above $46\%$ is out of reach for those eleven however well their confidence discriminates, and
Qwen-32B, the best of them at ranking its own answers, is capped at $43\%$. A conformal selector cannot repair the shortfall either: choosing among
the same top-1 predictions, it is subject to the same counting bound (App.~\ref{app:audit};
\citealp{hanselle2025conformal}).

At the frontier the ceiling lifts and the second direction weakens: a tail base rate of $15.8$--$22.4\%$
takes the bound to $\approx\!1.0$, which makes the target attainable rather than showing it is attained. It is:
at matched $10\%$ coverage the model's own confidence and the external gate are not separable on the three
frontier configurations ($+2.1$pp to the external one). The test earns its keep by disqualifying predictors rather than endorsing them
(App.~\ref{app:audit}).

\begin{table}[t]
\centering\footnotesize\setlength{\tabcolsep}{1.5pt}
\begin{tabular}{@{}lcccc@{}}
\toprule
ultra-rare tail & Qwen & Qwen & Qwen & Llama \\
                & 7B & 14B & 32B & 8B \\
\midrule
bare LLM, base Recall@1 & $1.7$ & $4.2$ & $4.3$ & $4.6$ \\
\quad gated by own conf.\ (own $c$) & $2$\,($66$) & $18$\,($16$) & $16$\,($16$) & $5$\,($92$) \\
\quad oracle ceiling at $10\%$ & $17$ & $42$ & $43$ & $46$ \\
\midrule
retriever margin at $10\%$ & $81$ & $74$ & $76$ & $80$ \\
\bottomrule
\end{tabular}
\caption{Selective accuracy (\%) on the ultra-rare stratum for four of the $16$ generative configurations in
Table~\ref{tab:collapse}. Parentheses on the second row give the coverage each LLM's own confidence actually
resolves, which is not $10\%$; the other two rows are at $10\%$, and the coverage-matched comparison is in
App.~\ref{app:calib}. The oracle row is $\min(1,p/c)$. The retriever row varies across columns only because
each model returned parseable output on a slightly different case set (App.~\ref{app:tabnotes}).}
\label{tab:method}
\end{table}
\subsection{Phenotype Retrieval on Ultra-Rare Cases}
\label{sec:reframe}
Is the tail simply hard? A classical information-content-weighted phenotype-overlap retriever, given the same
HPO terms, reaches $25.6\%$ on the ultra-rare tail after removing same-source curation leakage: we drop any
gold term supported \emph{only} by the case's own source publication (App.~\ref{app:decon}). The tail is therefore not uniformly hard: a ranker built directly on
phenotype--disease compatibility keeps far more accuracy there than the LLM, and the ordering survives
per-disease averaging. Neither is accurate enough to diagnose autonomously, which is why the deployable contribution is triage
rather than an accuracy chase.

The retriever's absolute accuracy is only \emph{partially identified}: curation is shared between the
records and the knowledge base, and a stricter same-publication control gives a sensitivity range of
$[5.2\%,25.6\%]$, where the ordering holds against the five weaker small models but not the frontier. What is stable across it is the lift over the gate's own base rate, $2.4$--$3.3\times$: a triage lift, not a
certified operating point, and three external cohorts corroborate its upper half (App.~\ref{app:decon},
\ref{app:external}).

\subsection{Top Scores versus Top-Two Margins}
\label{sec:method}

The claim is about how a score is built, not about medicine, so it should hold on a clinical prioritizer
and a text retriever alike; we take the clinical ones here and the text ones in \S\ref{sec:transfer}.

Exomiser \citep{smedley2015exomiser} shows it on a tool clinicians run. Its ultra-rare base Recall@1 is
$13.3\%$ over $n{=}939$ cases, and the gap between its first two candidates, from the same output, selects a
$10\%$ band at $29.0\%$ (disease-clustered $17.7$--$42.7\%$). Read its top score and no band can be drawn: the decile falls inside one tie
block, so its accuracy moves over $[0.0,4.3]\%$ on how ties are broken. Neither that block nor the tool's own
shipped $p$-value accounts for the difference, so this is not a straw target
(App.~\ref{app:exomiser}, Fig.~\ref{fig:rc};
\citealp{cooperstein2025exomiser,vestito2024reinterpretation}). Without variant data nothing breaks a tie once
the phenotype score saturates, so we claim the dissociation for phenotype-only prioritization and no
further.

Our own retriever separates the two further, gating at $45.8\%$ on its raw score against $81.0\%$ on its
margin, and only the delivered system is a deployment claim (Table~\ref{tab:method}): triaged by that
margin it answers the most confident $10\%$ of ultra-rare cases at $74$--$81\%$, where the bare LLM at the
coverage its own confidence resolves reaches $2$--$18\%$. The ordering survives simulated prospective phenotyping, where dropping the
most informative terms costs the retriever's band far less than it costs the LLM's (App.~\ref{app:robust}),
and it is retriever-first rather than hybrid because neither a heuristic combiner nor a validation-trained
one improves on the retriever alone (App.~\ref{app:fusion}).

\subsection{Testing the Shared-Component Mechanism}
\label{sec:scope}
Eq.~\eqref{eq:margincancel} is exact and carries no evidence on its own; the reverse manipulation does.
Subtracting an estimate of $b(x)$ from our retriever's scores lifts its raw gate from $45.8\%$ to $80.3\%$,
the margin's own level, though only with a post-hoc temperature and only for estimates read off the candidates
in contention, input-derived ones reaching at most $70.7\%$ (App.~\ref{app:margin_mechanism}). What no manipulation delivers is a label-free rule for
deciding in advance when the subtraction will help.

LIRICAL \citep{robinson2020lirical} shows why on a deployed tool. Reading one ranking in the two units it
ships moves the margin's advantage from $+0.056$ under the likelihood ratio to $-0.281$ under the post-test
probability: one monotone map, a swing of $0.337$, and nothing in the scores to choose between them. Across seven systems the advantage runs from $+0.310$ to $-0.026$, an ordering and not a uniform gain
(App.~\ref{app:margin_ranklimit}, \ref{app:transfer}).

The controlled version of that swing makes the rest concrete. Fourteen units of \emph{one} ranking, from the fusion literature's normalizations to an injected per-input
level at five strengths, leave Recall@1 identical at $77.6\%$ and the margin's AUROC inside
$0.738$--$0.804$ while the raw score's runs $0.307$--$0.822$, a spread eight times as wide.
Under a purely additive shift the margin's AUROC is unchanged to machine zero, as
Eq.~\eqref{eq:margincancel} requires, and so is the top score against the mean of the field: the injection
arms sweep the between-input share from $0.237$ to $0.988$ and cost the raw score $0.292$ of AUROC while
moving either contrast by at most $3\times10^{-6}$. The sweep therefore confirms the predicted invariance in an implemented scorer and shows the raw score's
collapse to be driven by exactly the component the contrasts remove. How far the shared level tracks the gain across these arms is
measurable, but only under an assumption about that level which Prop.~\ref{prop:unident} says the scores
cannot certify (App.~\ref{app:normsweep}).

\subsection{Different Confidence Signals for Correctness and Candidate Presence}
\label{sec:transfer}
Concept normalization is the closest neighbour and its survey's rules threshold the top-1 score
\citep{sevgili2022neuralel}. On the $12{,}741$ BC5CDR mentions the two are not
distinguishable ($-0.027$, $[-0.075,+0.021]$ clustered on mention strings), consistent with $C\!=\!0.011$.
Swapping the scorer for BM25 raises $C$ to $0.441$ on the same mentions and moves the margin from tied to
ahead ($0.743$ to $0.793$, $[+0.017,+0.082]$), a crossover registered before running: the scorer's
normalization, not the presence of text, decides whether the correction is needed.

Both runs ask only whether the top-1 is right. Deployed linkers must also decide whether to link at all, and
Eq.~\eqref{eq:margincancel} sends the two decisions to opposite parts of one score: correctness to the
\emph{difference}, an out-of-base mention to the \emph{level} it cancels. Masking a quarter of MEDIC makes $25.8\%$ of mentions unlinkable and asks both. Our recommendation is bounded here: under SapBERT the margin is the worse NIL detector at every rate
($-0.054$, $[-0.088,-0.023]$), and on MedMentions it is best in none of six cells. The field's rule is bounded by the same term: reading presence off the raw score works only where that
score is normalized, and under BM25 it must be corrected from outside the candidate set. The query's own IDF mass does it ($+0.109$, $[+0.053,+0.163]$), while the contrast that
estimates the same level from the contenders is the worst of four statistics at every rate
($0.636$--$0.657$ against $0.767$--$0.775$), behind even the raw score at the heaviest masking.

The split is not a property of short-mention linking. On SciFact the same reading holds for correctness,
ColBERTv2 gating at $0.764$ by its own MaxSim against $0.836$ by its margin, with a case-level share of
$C\!=\!0.374$ that is not query length. Under the same masking protocol, at every rate that margin is the
\emph{worst} of three for whether a relevant abstract survives
($0.647$--$0.696$ against $0.696$--$0.738$) while beating its raw score on top-1 correctness ($+0.043$ to
$+0.080$). That is the statistic adaptive-RAG gates threshold to decide
whether retrieval is needed at all \citep{wang2025targ} (App.~\ref{app:transfer}).

\section{Conclusion}
Selective prediction cannot rescue a ranker that rarely places the correct answer first. On ultra-rare
diseases, eight small open-weight LLMs lack enough correct top-ranked predictions to reach $50\%$ selective
accuracy at $10\%$ coverage, regardless of calibration; more accurate medical-specialised and frontier LLMs
pass the same feasibility check. Phenotype rankers provide a stronger starting point, but their scores expose
another distinction. For fixed candidate sets with unnormalized scores, the top-two margin can better
indicate whether the leader is correct by removing variation shared across candidates. That removal can also
discard information needed to tell whether the list contains an answer, as our SciFact and entity-linking
experiments show. Selective prediction over ranked outputs should therefore proceed in order: first test
whether base accuracy makes the target feasible, then choose a confidence signal for the decision being made.
Since unlabelled scores cannot reveal the gain from switching signals, the final choice requires labelled
validation.
\label{endmain}

\section*{Limitations}
Our thresholds are retrospective estimates, not guarantees for new patients. Certifying $80\%$ accuracy at
$10\%$ coverage would require about $26{,}600$ ultra-rare cases. Disease clustering reduces the effective
sample further, and conformal calibration misses its promised error rate on diseases absent from the
calibration data (App.~\ref{app:calib}). Because clinical sites will encounter such diseases, the reported
thresholds are not clinically certified.

The feasibility result depends on the delivery rule. We analyze top-1 endorsement at a fixed accuracy target
and coverage. If success means that the correct diagnosis appears anywhere in the top $k$, Recall@$k$
replaces Recall@1 and the ceiling rises. At $50\%$ accuracy and $10\%$ coverage, all eight small models are
ruled out at $k{=}1$, but only four at $k{=}5$ (App.~\ref{app:topk}). Passing the test means only that the
target is not ruled out, not that an available confidence score reaches it. This changes which models fail
the feasibility check, but not our model comparisons or margin analyses, which hold $k$ fixed.

The retriever's absolute accuracy remains uncertain because patient records and disease profiles share
curation. Leave-source-out removes direct same-publication overlap but not broader curation effects, and our
two controls yield a five-fold range in base accuracy (App.~\ref{app:decon}). The benchmark phenotypes were
also curated after diagnosis and may be richer than prospective inputs; feature removal is only a proxy for
that difference (App.~\ref{app:robust}). The margin's lift over the retriever's own base rate survives these
checks, but its absolute deployment accuracy is not established. Finally, the mechanism is directly tested on
one retriever; six other rankers provide supporting evidence but not the same causal test. New systems still
require labelled validation, which Prop.~\ref{prop:unident} shows the scores alone cannot substitute for.

\section*{Ethics Statement}
A triage gate can harm. On the answered decile one global rule delivers $71.0\%$, disease-clustered interval
$[58.8,79.9]$, so about three in ten answered patients receive a confidently wrong top-1 and four in ten is
inside the planning range. That is not a null event: it can trigger confirmatory testing, cascade testing of
relatives and counselling. Most of those errors still point at the right work-up, $73.0\%$ at the causative gene, and we grade them
rather than count them (App.~\ref{app:errgrade}). Grading is not reassurance: seven of the fifteen
near-misses pair a purely dominant entity with a purely recessive one, so the assay is right while the
recurrence risk, the cascade-testing targets and the reproductive counselling are wrong, and both gene-level
and disorder-level grading score that as a near miss. A third of the answered decile is of unknown
prevalence, a group a clinic cannot identify before diagnosis and on which accuracy is lower still. We
intend the gate as a ranking aid for a specialist, never as an autonomous decision, which would also engage
FDA SaMD and EU MDR. The measurements above are ordinary selective prediction, where a deferred case
is scored as unanswered. The deployment we propose is weaker than that metric: a deferred patient still
receives the differential the system would have produced anyway, without the confidence flag, so the gate
annotates rather than gate-keeps. It does not follow that deferral costs nothing. Withholding a model's output from
clinicians raises missed diagnoses relative to showing no model at all, because an absent prediction is read
as evidence of absence rather than as neutrality \citep{jabbour2025limits}. An unflagged case is therefore not
a null intervention, and how the absence of a flag is displayed is a design question prior to where the
threshold sits.

The gate is also not equally available, which is a documented hazard of selective classification rather than
a quirk of ours \citep{jones2021selective}. Of the $217$ sampled patients with one or two recorded findings,
none is ever flagged, against a flag rate of $24.6\%$ above fifteen terms, and yet the retriever is barely worse on them than on
patients with eight to eleven terms, $19.4\%$ against $24.0\%$. What collapses is not whether the system helps those patients but whether they
can qualify for the marker, because the margin grows with how many terms were recorded. That closes the gate on
patients at first presentation, those seen by generalists, and those in systems where deep phenotype coding is
not routine. App.~\ref{app:margin} measures what removing the dependence would cost, $70.7\%$ against $81.0\%$, and
App.~\ref{app:whichcases} characterises the patients the gate does select: an efficiency and equity
trade-off we have measured but not resolved.

All patient data is secondary use of de-identified, already-published records, and we use only structured HPO
terms and gold labels. We did not recruit or interact with patients, and obtaining consent was therefore not
ours to do: whatever consent the original case reports rest on was obtained by their authors, outside this
study, and the resources that aggregate them do not consistently record it. We note that rather than treat
public availability as consent. Re-identification risk is nonzero for any rare-disease case report, since an unusual combination of findings
can be distinctive, and prior publication does not remove it. Releasing source identifiers makes those reports
easier to locate, which is a risk we add rather than one we inherit; we judge it warranted for
reproducibility, release no case-report text or further patient attributes, and ask that the artifacts be
used under the original sources' terms. App.~\ref{app:decon} quantifies the concentration of curation in
one contributor account, $98.7\%$ of the records, which is a statement about a corpus-construction process and a
limitation of our own estimates, not a criticism of work this paper wholly depends on.

\bibliography{references}

\appendix

\section{Terminology}

\label{app:terms}
We use several terms from clinical genetics and rare-disease diagnosis. The definitions below state how each
term is used in this paper; they are not intended as complete clinical definitions.

\begin{itemize}[leftmargin=1.2em,itemsep=2pt,topsep=2pt,parsep=0pt]
\item \textbf{Phenotype.} An observable clinical feature of a patient, such as seizures or short stature, as
opposed to the underlying genetic cause.
\item \textbf{Human Phenotype Ontology (HPO).} A controlled vocabulary of approximately $18{,}000$ phenotypic
abnormalities arranged in an \emph{is-a} hierarchy. For example, ``absence seizure'' is a descendant of
``seizure''.
\item \textbf{Phenopacket.} A machine-readable patient record containing the HPO terms observed in that
patient and the confirmed diagnosis. We use the structured HPO terms as input; the original free-text case
report is not provided to the models.
\item \textbf{Differential diagnosis.} A ranked list of diseases that could explain a patient's findings. Our
systems return a top-5 differential, and Recall@1 and Recall@5 measure whether the confirmed diagnosis appears
first or anywhere in that list.
\item \textbf{OMIM.} Online Mendelian Inheritance in Man, a reference catalogue of Mendelian diseases. OMIM
identifiers define the disease label space in our experiments: the retriever ranks $8{,}553$ candidate entries,
and the evaluation set contains $780$ gold diseases.
\item \textbf{HPOA.} The annotation resource linking each OMIM disease to the HPO terms reported for it,
together with the publications supporting those annotations. HPOA is the entire knowledge base of our
phenotype retriever.
\item \textbf{Orphanet prevalence.} An epidemiological classification of rare diseases, defined independently
of any language-model training corpus. We use it to form the ultra-rare, less-rare and
undocumented-prevalence strata.
\item \textbf{Ultra-rare.} The union of the two rarest Orphanet prevalence classes: $<\!1$ case per million
and $1$--$9$ cases per million.
\item \textbf{Prioritizer.} A non-LLM system that ranks candidate diseases or genes by their compatibility
with the patient's findings. Exomiser is the clinically deployed example used in this paper.
\item \textbf{Phenotype-only prioritization.} Disease or gene prioritization from HPO terms without
genetic-variant evidence. Our Exomiser experiment uses this setting and should not be read as an evaluation of
its full genome-aware pipeline.
\item \textbf{Information content (IC).} The negative logarithm of a phenotype term's frequency across
diseases. Rare and specific terms receive more weight than common terms when the retriever compares a patient
with a disease profile.
\item \textbf{True-path propagation.} Adding the ancestors of each HPO term to a disease profile, so that a
specific patient finding can match a disease annotated only with a broader ancestral term.
\item \textbf{Curation leakage.} Overlap created when a patient's phenotypes and the corresponding disease
profile were curated from the same publication or by a shared curation process. Such overlap can inflate
phenotype-retrieval performance even though no case text is given to the retriever.
\item \textbf{Selective prediction.} A system answers only a subset of cases and defers the rest;
\emph{coverage} is the fraction answered and \emph{selective accuracy} is accuracy on that subset
(\S\ref{sec:framework}).
\end{itemize}

\section{Margin Derivations and Mechanism Tests}

\label{app:margin}
The main text gives the operational argument for using the top-two margin. This appendix separates three
claims that require different kinds of support. First, the margin removes a candidate-independent case-level
score shift; this is algebra. Second, the gain from doing so cannot be identified from candidate ranks alone;
this is a negative result with one decisive counterexample and one empirical refutation. Third, invariance to
the shift does not by itself guarantee that a confidence signal will work; this needs measurement. We then
test the proposed mechanism on our retriever and derive a scale-specific prediction for Exomiser.

\subsection{Score Representation and Invariance}
\label{app:margin_invariance}
Let a prioritizer assign scores $s(d\mid x)$ to candidates $d\in D$ for case $x$. Treating these as
unnormalized log-potentials induces
\begin{equation}
P_T(d\mid x)=\frac{\exp(s(d\mid x)/T)}{Z_T(x)},
\label{eq:app_gibbs}
\end{equation}
with normalizer $Z_T(x)=\sum_{d'\in D}\exp(s(d'\mid x)/T)$,
for a temperature $T>0$. With $d_1,d_2$ the highest- and second-highest-scoring candidates and scores
$s_1,s_2$, their margin satisfies
\begin{equation}
s_1-s_2=T\log\frac{P_T(d_1\mid x)}{P_T(d_2\mid x)},
\label{eq:app_margin_odds}
\end{equation}
or equivalently $P_T(d_1\mid x,\{d_1,d_2\})=\sigma\!\left((s_1-s_2)/T\right)$ after restricting to those two
candidates. The margin is thus proportional to the log-odds of the leader against its closest competitor, and
the normalizer $Z_T(x)$ cancels.

Now consider
\begin{equation}
s(d\mid x)\longmapsto \alpha s(d\mid x)+\beta(x),\qquad \alpha>0,
\label{eq:app_gauge}
\end{equation}
with $\alpha$ a global rescaling and $\beta(x)$ an arbitrary case-level offset shared by all candidates. Then
$s_1\mapsto\alpha s_1+\beta(x)$ while $s_1-s_2\mapsto\alpha(s_1-s_2)$. The top score moves with $\beta(x)$ and
may reorder cases; the margin removes $\beta(x)$, and the remaining positive factor changes neither its
ordering of cases nor its risk--coverage curve.

This is the conditioning argument behind conditional logit \citep{mcfadden1974conditional} and partial
likelihood \citep{cox1972regression}, and related invariance arguments motivate margin-based selective
classification \citep{liang2024selective} and likelihood-ratio rejection rules \citep{hengsoh2025rlog}. Our
contribution is not the algebra but its use in identifying the failure of top-score gating in the ultra-rare
regime.

\subsection{A Deployed Instance, and the Limits of Our Evidence}
\label{app:margin_ranklimit}
Proposition~\ref{prop:unident} bounds what the unlabelled score distribution can determine, and rank
statistics are its special case. LIRICAL supplies a deployed witness rather than a constructed one. It reports
both a composite likelihood ratio and a post-test probability for one identical ranking of one identical case
set, and the pretest probability is a single value for every disease and case, so the second is a fixed
monotone map of the first. The margin's advantage nonetheless moves from $+0.056$ to $-0.281$, a swing of
$0.337$. The map is destructive because it saturates: composite likelihood ratios run past $10^{29}$, so the
posterior is pinned at the top of its range on $79.8\%$ of cases and its top-two difference is numerically zero
on $28.0\%$, while every rank is preserved. Dropping the $7$ cases where the map stops being strictly
increasing in double precision leaves the swing at $0.343$, so saturation is the whole of the effect. That
single swing is as large as the entire spread of the gain across our seven systems, which is also $0.337$; the
two are distinct quantities that happen to agree.

Those seven systems are less independent than the count suggests, which is why $C$ is reported as a measured
property and no ordering is drawn from them: LIRICAL and Exomiser run on the same $939$ cases, the two entity
linkers on the same mentions, and MaxSim SUM and MaxSim MEAN are one ranking in two units, the second being
the first divided by query length and agreeing on every top-$1$, itself a small instance of the same point,
since that one ranking has gains of $+0.178$ and $+0.077$.

The gap this addresses is real even so. Post-hoc normalization of a broken confidence estimator is known to
repair selective classification on vision classifiers, but the repair is fitted on labelled data and offers no
way to tell in advance which model needs it \citep{cattelan2024broken}; likelihood-ratio analysis gives the
condition under which a top-two statistic is optimal, that the distribution concentrate on the leading pair,
but not a test of it \citep{hengsoh2025rlog}. Calibrating a label-free test of that condition would need
scorers spanning the saturation range, which our seven do not.

\subsection{One Ranking, Fourteen Units}
\label{app:normsweep}
Every transform below is strictly increasing within a mention, so the ranking, the top-1 and Recall@1 are
identical in all fourteen arms; only the score domain moves. This is the controlled counterpart of LIRICAL's
two vendor units, inside one system, and it replaces the seven-system correlation the paper draws no ordering
from.

\begin{table}[t]
\centering\footnotesize\setlength{\tabcolsep}{4pt}
\begin{tabular}{@{}lcccc@{}}
\toprule
score domain & $C$ & raw & margin & gain \\
\midrule
cosine         & $0.237$  & $0.822$ & $0.796$ & $-0.026$ \\
center         & $-0.111$ & $0.790$ & $0.796$ & $+0.006$ \\
scale          & $0.974$  & $0.307$ & $0.738$ & $+0.431$ \\
\midrule
softmax $T{=}1.0$   & $-0.111$ & $0.791$ & $0.798$ & $+0.007$ \\
softmax $T{=}0.1$   & $-0.111$ & $0.802$ & $0.804$ & $+0.002$ \\
softmax $T{=}0.02$  & $-0.111$ & $0.800$ & $0.798$ & $-0.002$ \\
minmax         & $-0.047$ & $0.500$ & $0.738$ & $+0.238$ \\
zmuv           & $-0.111$ & $0.739$ & $0.738$ & $-0.001$ \\
sum            & $-0.111$ & $0.704$ & $0.739$ & $+0.036$ \\
\midrule
inject $0.5$   & $0.388$  & $0.773$ & $0.796$ & $+0.023$ \\
inject $1$     & $0.617$  & $0.716$ & $0.796$ & $+0.080$ \\
inject $2$     & $0.847$  & $0.635$ & $0.796$ & $+0.161$ \\
inject $4$     & $0.955$  & $0.569$ & $0.796$ & $+0.227$ \\
inject $8$     & $0.988$  & $0.531$ & $0.796$ & $+0.265$ \\
\bottomrule
\end{tabular}
\caption{Fourteen units of one ranking on $12{,}741$ BC5CDR mentions (AUROC); Recall@1 is $77.6\%$ in every
row. \emph{center} shifts each mention's scores to zero mean, \emph{scale} divides by their standard
deviation, \emph{inject} adds a per-mention offset of the given size drawn independently of correctness. $C$
here is the one-way ANOVA over the ten contenders, not the top-two form \S\ref{sec:transfer} quotes; on the
untransformed cosine the two read $0.237$ and $0.011$ on the same mentions.}
\label{tab:normsweep}
\end{table}

Table~\ref{tab:normsweep} separates the two halves of the decomposition. Under \emph{center}, a purely
additive per-mention shift, the margin's AUROC is $0.796100$ against the untransformed $0.796100$, equal in
double precision as Eq.~\eqref{eq:margincancel} requires, while the raw score moves from $0.822$ to
$0.790$. Under \emph{scale}, a purely multiplicative one, the margin does move, which the algebra does not
forbid. The injection arms then sweep the between-input share from $0.388$ to $0.988$ with the margin's AUROC
fixed at $0.796$ throughout: the gain grows from $+0.023$ to $+0.265$ entirely because the raw score decays
from $0.773$ to $0.531$, not because the margin improves.

One observation is worth stating on its own. The normalizations that fix a per-row statistic outright, softmax at
every temperature, $z$-score and sum, drive the between-input variance to zero by construction, which is why
$C$ sits at the estimator's $K{=}10$ floor of $-1/9$ in those five arms and in the centering arm. Min-max
fixes only the two extremes, so its row means still vary a little and it reads $-0.047$. Those transforms can sweep the shared level down and no further, so a family built only from
them cannot test whether the share tracks the gain; the injection arms exist to span it upward. It also means
that a system whose scores have already been normalized this way has no shared level left for the margin to
cancel, which is the same reading the normalized cosine gets in App.~\ref{app:transfer}.

\subsection{Mechanism Test on the Retriever}
\label{app:margin_mechanism}
Offset invariance is necessary for the proposed explanation but not sufficient to select a useful signal. The
mean-centered top score $s_1-|D|^{-1}\sum_{d\in D}s(d\mid x)$ has the same invariance to
candidate-independent shifts as the margin, yet reaches only $47.9\%$ accuracy at $10\%$ coverage against the
margin's $81.0\%$. The invariance argument identifies an admissible class of signals; it does not show that
every member of that class works.

We test the cancellation on our retriever because it is the only system whose full candidate-score vector we
can recompute. The experiment uses the complete ultra-rare stratum ($n{=}4{,}780$, base Recall@1 $25.4\%$) and
gates at $10\%$ realized coverage.

\begin{table}[t]
\centering\small
\resizebox{\columnwidth}{!}{\begin{tabular}{@{}lc@{}}
\toprule
confidence signal & accuracy at $10\%$ coverage \\
\midrule
raw top score $s_1$                          & $45.8\%$ \\
mean-centered top score                      & $47.9\%$ \\
normalized log-score $s_1/T-\log Z_T$, best $T$ & $80.3\%$ \\
\textbf{top-two margin $s_1-s_2$}            & $\mathbf{81.0\%}$ \\
softmax entropy, full score vector           & $81.4\%$ \\
\bottomrule
\end{tabular}}
\caption{Testing the cancellation mechanism on the ultra-rare retriever output; for entropy, lower is treated
as more confident. The temperature of the normalized log-score is selected after observing performance,
whereas the margin has no fitted parameter.}
\label{tab:margin_mechanism}
\end{table}

The normalized log-score removes the per-case normalizer and recovers almost all of the margin's advantage,
which supports the mechanism without being an equally deployable replacement: it ranges from $34.9\%$ to
$80.3\%$ as $T$ varies over $[0.25,50]$, best at $T{=}3$, and that temperature is chosen after the fact. The
family also explains why the margin works. As $T\!\to\!0$ the normalized log-score approaches a ranking
determined by the top-two margin; as $T\!\to\!\infty$ it approaches the mean-centered top score. The
low-temperature region is numerically unstable, since at $T{=}0.5$ some $60.5\%$ of cases collide because
the relevant exponential differences underflow, so the limit should not be read as an empirical temperature
result.

\subsection{Locating the Shared Score Component}
\label{app:margin_component}
Is the removable component determined by the case input or by the candidate-score distribution? Input-derived
quantities (the number of recorded HPO terms, the number matched, and summaries of their information content)
explain $R^2\!=\!0.92$ of the raw top score's variance, and the raw score correlates with matched-term count
at $\rho\!=\!0.93$, so the raw score strongly reflects case composition. Removing that dependence does not recover the margin's gate: the best input-derived normalizer
reaches $70.7\%$ and residualizing on the in-sample optimal linear combination of all input features reaches
$66.5\%$, against $81.0\%$ for the margin.

Score-derived normalizers behave differently. Subtracting the score at increasing rank weakens the gate
gradually: rank-10 gives $74.1\%$, rank-100 gives $62.3\%$, and subtracting the median gives $47.1\%$,
close to the uncorrected top score. The relevant shared level is therefore set by the few candidates genuinely
in contention, not by the bulk of all $8{,}553$ candidates nor by the patient description alone, which is why
a head-dominated quantity such as $\log Z_T(x)$ succeeds where the mean does not. This refines
Eq.~\eqref{eq:scoredecomp}: the useful approximation to the shared component is case-specific but is expressed
through the head of the competitor distribution, and the margin estimates that level locally using the nearest
competitor.

\subsection{Scope of the Argument}
\label{app:margin_scope}
The log-potential reading is an assumption, not a property guaranteed for every scoring system. The weaker
requirement the cancellation argument actually needs is that candidates for one case share an approximately
candidate-independent additive component; a difference removes such a component even when the scores are not
calibrated probabilities. The margin must also be taken on a scale that is additive over evidence: it is
stable under a global positive affine map but not under an arbitrary monotone one, as LIRICAL shows, so
choosing the score domain is part of specifying the signal. Finally, the argument needs a fixed candidate set
scored jointly by one predictor; independent samples from a free-form generator are not the leading candidates
of such a set and need not have a meaningful runner-up (App.~\ref{app:transfer}).

The mechanism experiment is conducted on one retriever whose complete score vector is available; Exomiser
tests one prediction about score scale, and the remaining systems provide transfer evidence rather than
further interventions. We therefore treat cancellation as the supported mechanism in our main setting, not as
a theorem that the margin must beat the top score for every ranker. Determinism and floating-point sensitivity are recorded in App.~\ref{app:repro}.

\section{Prevalence Strata and Averaging}

Two questions about how the collapse is measured: whether the prevalence labels and their binning create it, and whether per-case averaging inflates it.

\subsection{Prevalence Stratification}
\label{app:strat}
\textbf{Bin correction.} Orphanet lists some diseases with prevalence class ``unknown''/``not yet documented.''
Pooling these (truthy strings) into the less-rare bin would make $57\%$ of that anchor
undocumented-prevalence; only $372/780$ gold diseases carry a real class. We route undocumented cases to a
separate \emph{unknown} bin. Because the LLM performs near the tail on undocumented diseases, that routing
raises the less-rare anchor ($\sim\!19\%\!\to\!\sim\!37$--$41\%$) and \emph{widens} the reported collapse.

\textbf{Finer dose-response.} Fig.~\ref{fig:dose} plots Recall@1 across the ordinal Orphanet classes. LLM
accuracy \emph{peaks} at the moderately-rare $1$--$9/10^5$ class and then declines monotonically across the two
ultra-rare classes to a near-zero floor (e.g.\ Qwen2.5-VL-32B $45.8\%$ at $1$--$9/10^5 \to 18.2\%$ at
$1$--$9/10^6 \to 0.4\%$ at ${<}1/10^6$; that model's own $n{=}166/198/703$, the class sizes $924/993/3787$
being the denominators of the retriever curve below), confirming a genuine gradient/cliff rather than a
two-bin artifact; the decontaminated retriever is nearly flat over the same classes
($34.0\!\to\!26.7\!\to\!25.2\%$). The dip at the \emph{more-common} but sparsely-populated $1$--$5/10^4$ class
($n=243$; more common than the $1$--$9/10^5$ peak) is within noise; the non-monotonicity is at the common end and
does not affect the ultra-rare collapse. A point-prevalence (rather than rarest-class) rule
leaves the collapse magnitude essentially unchanged.
\begin{figure}[h]
\centering\includegraphics[width=\columnwidth]{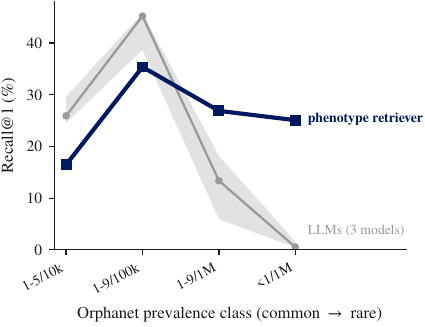}
\caption{Recall@1 across ordinal Orphanet prevalence classes (common$\to$rare). LLM accuracy peaks at the
moderately-rare class then collapses across the ultra-rare classes; the decontaminated phenotype retriever
(black) stays nearly flat.}
\label{fig:dose}
\end{figure}

\subsection{Per-Case and Per-Disease Averaging}
\label{app:macro}
The prevalence bins are dominated by a few high-frequency, textbook-famous diseases. The less-rare bin holds
$1{,}167$ cases but only $81$ diseases; Neurofibromatosis-1 (OMIM:162200) alone is $405$ cases ($34.7\%$ of the
bin, retriever Recall@1 $43\%$), and the top-5 diseases are $59.6\%$. The ultra-rare bin is far less concentrated
($291$ diseases; largest $9.7\%$). Per-case (micro) averaging therefore weights the famous less-rare
diseases heavily and inflates both the anchor and the collapse ratio. Per-disease (macro) averaging removes this
weighting (Table~\ref{tab:macro}): the less-rare LLM anchor falls from $\sim\!40\%$ (micro) to $\sim\!10\%$
(macro), the LLM collapse ratio falls from $8.6$--$25.1\times$ (micro) to $1.9$--$6.1\times$ (macro), and the
retriever's per-case decline ($31.4\%\!\to\!25.6\%$, itself not significant) \emph{reverses} to a per-disease \emph{rise}
($12.7\%\!\to\!23.3\%$).

Two things follow, and they should be kept apart. First, the \emph{magnitude} of the collapse is a micro
effect: per-disease the LLM fall is only $\sim\!2$--$6\times$. Second, the tail \emph{contrast} is not. On the
ultra-rare stratum, which is not concentrated, the retriever holds $\sim\!24\%$ under both averagings while the
small models reach $0.7$--$4.6\%$ per case and $1.3$--$4.9\%$ per disease. The deployment recommendation rests
on that contrast and not on the height of the less-rare anchor.

\begin{table*}[t]
\centering\small
\begin{tabular}{@{}lcccc@{}}
\toprule
 & \multicolumn{2}{c}{less-rare} & \multicolumn{2}{c}{ultra-rare} \\
\cmidrule(lr){2-3}\cmidrule(lr){4-5}
 & per-case & per-disease & per-case & per-disease \\
\midrule
LLM (8 small models)$^{*}$ & $7$--$41\%$ & $\sim\!10\%$ & $0.7$--$4.6\%$ & $1.3$--$4.9\%$ \\
phenotype retriever       & $31.4\%$     & $12.7\%$     & $25.6\%$       & $23.3\%$ \\
\bottomrule
\end{tabular}
\caption{Recall@1 under per-case (micro) and per-disease (macro) averaging. $^{*}$ranges span all eight
small models; Yi, weak in both strata at $7\%$ even on less-rare, is the low outlier. Retriever micro figures
use the full decontaminated set; its macro figures and every LLM figure use the $N{=}2000$ sample.}
\label{tab:macro}
\end{table*}

\textbf{Prevalence against corpus exposure.} Prevalence here is collinear with pretraining
frequency and fame, since the famous less-rare diseases are also the high-accuracy ones, so we do not claim to
separate the two \citep{kandpal2023longtail,sun2024headtotail}. That binning a benchmark by an external
popularity variable exposes a hidden tail is itself established \citep{sun2024headtotail}; what is specific
here is \S\ref{sec:reframe}, where the discriminating evidence is present in the input, so this is a failure to
\emph{use} evidence rather than an exposure gap. \citet{chen2026rarearena} run their own Orphanet-mapped
subgroup analysis and report no decline with prevalence, attributing their gradient to PubMed mention counts
instead. We take that seriously: their cases are PMC case reports drawn from the same literature the models
are pretrained on, so case availability there is decoupled from epidemiological rarity by construction and
their prevalence axis is not ours. Neither study identifies prevalence net of frequency. We cannot measure
pretraining exposure, so we hold a proxy fixed instead: the number of cases of that disease in this corpus, which is what curator attention
and the published literature jointly produced. Two estimators, since neither alone convinces. Binning diseases
by exact corpus frequency and comparing strata \emph{within} each bin assumes no model of the frequency effect
and leaves the per-disease gap at $+0.151$ against a raw $+0.159$, i.e.\ $95\%$ retained. Matching each
less-rare disease to the ultra-rare disease nearest in log frequency, without replacement and within a caliper,
gives $+0.169$, i.e.\ $107\%$. The residual is positive for all $7$ models under both. The matched design is
small ($13$ pairs, minimum detectable difference $0.118$). These controls therefore bound the exposure
explanation rather than identifying a prevalence effect independent of frequency.

\section{Curation Leakage and External Validation}

The retriever's absolute accuracy is partially identified. This section defines the two controls that bracket it, measures what neither can see, and prices the residual from outside the corpus.

\subsection{Leave-Source-Out Decontamination}
\label{app:decon}
Phenopacket case identifiers carry their source publication (PMID), and each HPO annotation in HPOA carries the
PMID(s) supporting it. The fraction of cases whose own source PMID also appears as a source of the gold disease's
HPOA profile is $74.5\%$ (less-rare), $60.6\%$ (ultra-rare), and $89.9\%$ (unknown-prevalence): for most cases the
retriever matches a patient against a profile transcribed in part from the same publication. To decontaminate,
for each case we drop from the gold profile any HPO term whose \emph{only} supporting PMID is that case's own
source (a conservative leave-source-out keeping independently corroborated terms), re-propagate over the
ontology, and re-score with the identical tie-aware expected Recall@1; all other diseases are untouched.
Recall@1 falls: less-rare $41.9\%\!\to\!31.4\%$, ultra-rare $49.4\%\!\to\!25.6\%$, unknown $67.5\%\!\to\!18.0\%$
($N{=}10{,}345$; the $29$ cases whose gold disease carries no HPOA profile are absent from the candidate set and
are excluded from every retriever score, decontaminated or not). A tie-tolerant double-precision reimplementation
gives $28.1/24.4/16.6\%$; we quote that range where the exact value matters, as float-exact tie detection moves
the less-rare figure between $26.5$ and $32.5\%$.

\textbf{Residual curator-level coupling.} $98.7\%$ of the phenopackets were created by
a single curator account, and that same curator contributed ${\sim}21{,}800$ HPOA
disease-phenotype rows, so patient record and knowledge base are not independently authored, and a PMID-based
control cannot see this. Of the ultra-rare patient terms matching the gold profile, $51.3\%$ were biocurated by
the person who wrote the patient record and $49.6\%$ by that person \emph{from the same publication}; our control
removes $23.0\%$ and leaves $26.6\%$. The terms that survive do so by ``independent corroboration'', but for
$91.2\%$ of them the corroborating publication itself contributes cases of the same disease to this corpus, and
for $29.9\%$ of ultra-rare cases the gold's HPOA frequency denominator exactly equals that disease's case count
here ($2.1\times$ a case-weighted permutation null; a joint disease-level permutation, the correct unit, gives
$3.1\times$, $p<0.001$), i.e.\ the annotation was tabulated from these patients. A strict control
that drops \emph{any} gold term the case's own publication supports gives ultra-rare $5.2\%$
(less-rare $6.1\%$).

\textbf{Partial identification of retriever accuracy.} The two
controls bound the estimand from opposite sides under a monotone-bias argument: the strict control discards terms
that other publications genuinely establish, so it is biased \emph{down}; leave-source-out retains terms written
by the record's own curator from other papers, so it is biased \emph{up}. Neither is the estimand, and HPOA
records PMIDs but not cohort identity, so nothing in the metadata closes the gap. We deliberately do \emph{not} invoke the contaminated-data bounds of
\citet{horowitz1995contaminated}: those require the contamination \emph{rate} to be known and leave the
contaminating distribution free, which is the reverse of our situation, we know the channel and not the rate, and
with the rate unrestricted their sharp bounds are vacuous. What we run is a sensitivity analysis indexed by that
unknown rate, in the sense of \citet{rosenbaum2002observational}, and we report the breakdown point at which each
conclusion fails \citep{mastenpoirier2020breakdown}. We accordingly
report $[5.2\%,\,24.4\%]$ as a \emph{sensitivity range}, not a confidence interval. The same range at the
published implementation's precision is $[5.2\%,25.6\%]$, the $1.2$pp being the tie-handling gap above, and
within-control sampling
error is separate, and disease-clustered at the upper end gives $24.4\%$ $[17.1,33.0]$. The strict end is not
knowledge-base deletion: gold profiles retain $67.2\%$ of their annotations and only $6.8\%$ of ultra-rare
profiles empty.

\textbf{Stability across the sensitivity range.} Sweeping the whole interval, the gate's \emph{lift} over its own base stays in $[2.4,3.3]\times$, it is
non-monotone, peaking at $3.3\times$ in the interior before falling to $2.4\times$ at the strict end, so the
interval must be swept and not merely evaluated at its endpoints, and its AUROC falls monotonically from $0.80$
to $0.65$. The margin continues to outrank the best predictor-free gate at every point, though its advantage
narrows from $1.9\times$ to $1.3\times$; we state that as a dominance, not as a level, because the
predictor-free gate's own level moves across the sweep too and we did not register it endpoint by endpoint. What is \emph{not} invariant is the absolute level, 
and with it the comparison against the frontier model, whose breakdown point is only ${\approx}6\%$ residual
contamination, so we rest nothing on it. The small-model comparison is not invariant either, and we do
not claim it is: at the strict endpoint the retriever's $5.2\%$ clears $5$ of the $8$ small models, but the
strongest three ($4.2$--$4.6\%$) are not separable from it at this sample size. We therefore rest the gate claim on its
lift over its own base rather than on the level, and state the accuracy comparison as holding under
leave-source-out and against the five weaker small models under every control. \emph{Two asymmetries matter when reading the bracket.} The LLMs are not decontaminated at all, the same case
reports are in their pretraining data, so a maximally-stripped retriever against an unstripped LLM compares a
lower bound with an upper bound rather than measuring which predictor is better. And at the strict end the
retriever falls below the frontier model, so ``no LLM surpasses it on the tail'' holds against the frontier only
under leave-source-out, and at the strict end is a statement about the five weaker small models. Note that the vocabulary ablations
below rebut the \emph{near-unique-key} mechanism and are computed off the leave-source-out retriever; they do
\emph{not} address shared authorship, which no vocabulary restriction can remove.

\textbf{Selection-based check and its confounds.} One can avoid the over-correction objection
entirely by \emph{selecting} rather than deleting: keep only tail cases for which no matched gold term is
supported by the patient's own publication. That leaves $2{,}441/4{,}780$ cases, on which the retriever falls
$25.4\%\!\to\!9.7\%$ while the frontier model barely moves ($22.4\!\to\!19.3\%$), i.e.\ on that subpopulation the
frontier model is ahead. Two confounds inflate this. Those cases have far thinner gold profiles ($40.2$ vs.\
$84.0$ HPOA terms), and the flag itself is mechanically tied to how much matched: $P(\text{flagged})$ rises from
$0\%$ at $1$--$3$ matched terms to $59\%$ at ${\geq}26$, while matching more terms independently predicts being
correct. Stratifying on matched-term count does not explain the gap away: it shrinks to $+11.1$pp among
cases matching between $10$ and $25$ terms ($n{=}1{,}223$), but in the ${\geq}26$ stratum, which holds $63\%$ of the
tail, it is $+34.8$pp against the $+32.3$pp unstratified gap. We therefore read this as corroborating the bracket, 
own-source support does carry real signal beyond case difficulty, rather than as a clean estimate of either
endpoint.

\textbf{A near-unique-key artifact does not operate here.} A canonical field that identifies an entity
outright, and does so more often for rarer entities, can manufacture a tail gradient on its own. Term
specificity does rise with rarity ($17.7\!\to\!24.9\%$ of diseases own a
globally-private term), but deleting \emph{every} such term from both patient and knowledge base leaves the tail
unchanged ($24.4\!\to\!24.4\%$, gate $79.7\!\to\!78.5\%$): private keys carry only $13.9\%$ of the tail's Recall@1
mass. Removing the patient's single highest-information term costs $4.4$pp, and restricting the vocabulary to
phenotypes shared by more than $100$ diseases still leaves the tail at $16.8\%$, above every small model. An HPO
profile is $10$--$40$ graded-specificity terms, not one canonical field, which is why that failure mode has
no analogue here. The gate is the most robust component: margin AUROC stays at $0.78$--$0.87$ across all of these
ablations and at $0.65$ under the strict control. Its \emph{operating point}, however, does not survive that
control: with a $5.2\%$ base the confident decile delivers $12.3\%$ (lift $2.4\times$). Note this configuration
\emph{passes} $p\!\ge\!\tau c$ ($5.2\%$ against a $5\%$ floor) while delivering nowhere near $\tau$, a clean
illustration that the criterion is necessary and not sufficient (\S\ref{sec:calib}), and that at the strict end
the gate would be a directional signal rather than a safe operating point. App.~\ref{app:external} is why we do
not read the strict end as the operating truth.

The apparent ``flat/higher on the tail'' shape (\S\ref{sec:reframe}) is thus partly curation leakage. We do
\emph{not} claim a direction for the decontaminated per-case gradient: the less-rare bin holds $1{,}167$ cases but
only $81$ diseases (neurofibromatosis type 1 alone is $34.7\%$ of it, the top five $59.6\%$), and a
disease-clustered bootstrap puts the gradient at $-3.7$pp with $95\%$ CI $[-14.5,+14.6]$, which per-disease
macro-averaging flips to $+12.8$pp (App.~\ref{app:macro}). Absolute retriever \% is an upper bound for two
\emph{distinct} reasons: same-source leakage, which we measure and remove here, and post-hoc HPO \emph{completeness},
a distribution shift decontamination does not touch and which we quantify separately (App.~\ref{app:robust}), not
because the retriever beats real tools (\S\ref{sec:bench}); all triage signals (\S\ref{sec:method}) are computed
on this decontaminated retriever.

\subsection{Evaluation on External Rare-Disease Cohorts}
\label{app:external}
App.~\ref{app:decon} leaves the tail retriever partially identified because both of our controls are proxies for
the same unobserved quantity: how much of the retriever's accuracy comes from one curator having written both the
patient records and much of the knowledge base. A PMID-based control cannot see curator-level coupling, and the
strict control over-corrects by construction, since a gold term the case's publication happens to mention is not
thereby a leaked term. No amount of further decontamination of \emph{this} corpus can separate the two.

We therefore price the leak from outside. RareBench \citep{chen2024rarebench} aggregates five rare-disease cohorts
none of which is Phenopacket Store; three are curated by groups with no relation to it: RAMEDIS
(Bielefeld University and the Reutlingen children's hospital), HMS (Germany), and MME
(CHEO/SickKids/Toronto). We exclude RareBench's fourth cohort, LIRICAL, on purpose: its cases derive from
published case reports inside the same Monarch/HPO ecosystem, so it is not a clean independence test. Applying
the paper's own prevalence rule (App.~\ref{app:strat}) to the remaining cases gives $243$ ultra-rare cases over
$63$ diseases. We run the \emph{identical} retriever, same code, same HPOA index, same IC weighting, no
refitting; Table~\ref{tab:external} reports it.

\begin{table*}[t]
\centering\footnotesize\setlength{\tabcolsep}{5pt}
\begin{tabular}{@{}llrrr@{}}
\toprule
cohort & curation source & $n$ & R@1 & clustered 95\% CI \\
\midrule
RAMEDIS & Bielefeld U.\ / Reutlingen children's hosp.\ (DE) & $177$ & $19.8\%$ & \\
HMS     & Germany                                          & $27$  & $0.0\%$  & \\
MME     & CHEO / SickKids / U.\ Toronto (CA)               & $39$  & $53.8\%$ & \\
\midrule
pooled  &                                                  & $243$ & $23.0\%$ & $14.7$--$32.7\%$ \\
\quad disease-disjoint & (diseases absent from Phenopacket Store) & $216$ & $19.4\%$ & $12.2$--$28.7\%$ \\
\bottomrule
\end{tabular}
\caption{The identical retriever on three rare-disease cohorts curated independently of Phenopacket Store,
with no refitting. Intervals are clustered on disease; the disease-disjoint row keeps only cases whose
diagnosis never appears in our own corpus.}
\label{tab:external}
\end{table*}

\textbf{External estimates and the sensitivity range.} Disease-clustered, the external interval sits at the
upper end; clustered at the \emph{cohort} level, the right level, since a deployment draws a cohort and not a
disease, it contains both endpoints, and a $\chi^2$ test rejects a common rate across the three
($p<10^{-6}$). We therefore report corroboration, not identification. Reporting the cohort we excluded makes the
same point: LIRICAL, whose cases come from published case reports inside the HPO ecosystem, scores $52.0\%$
($n{=}244$), twice the independent cohorts and close to the undecontaminated Phenopacket Store tail
($49.4\%$), which is what the leakage account predicts. The strict control is thus over-conservative rather than merely conservative: deleting every gold term
the case's own publication supports removes genuine phenotype signal, not only leakage. Restricting to the $216$
cases whose disease does not appear in Phenopacket Store at all, so that neither the case nor the disease is
shared with our corpus, gives $19.4\%$, still excluding the strict endpoint. Per-disease averaging gives
$29.2\%$, i.e.\ the estimate does not depend on case concentration.

\textbf{Limits of the external-cohort evidence.} Between-cohort spread ($0.0$--$53.8\%$) is far wider than the sampling
uncertainty within any one of them, so ${\sim}23\%$ is an estimate for \emph{this mixture} of independent cohorts,
not a universal constant; a deployment in a population resembling HMS should expect much less. HMS's zero is
partly structural, $3$ of its cases have a gold disease with no phenotypic HPOA annotation at all, hence
unretrievable by any phenotype method, but most of it is genuine failure. $n{=}243$ is small against the main
benchmark's $4{,}780$. Finally, RareBench's HPO codings were produced by its own authors; if that mapping
consulted HPOA disease profiles, a weaker second-order coupling survives that these files cannot test. We
therefore treat the external estimate as pricing the \emph{first-order} curator confound, which is the one that
made the interval five-fold wide, and continue to report the sweep in App.~\ref{app:decon} rather than replacing
it.

\section{Feasibility, Calibration and Delivery Rules}

What the feasibility bound forbids, what the incumbent confidence signals can actually resolve, and how both move with the delivery rule.

\subsection{Auditing the Feasibility Criterion}
\label{app:audit}
The criterion is one line of algebra, so its interest is in whether it \emph{discriminates} when applied
broadly. We read it across every (predictor, stratum) cell we have, ten LLMs by prevalence bin, our
retriever, Exomiser by its own score and by its shipped $p$-value, and the SapBERT and BM25 entity linkers by
concept frequency, at $c{=}10\%$.

\textbf{Most cells are non-binding.} In $21$ of the $29$ the base rate is high enough that the
ceiling is $1.0$, so no $\tau$ forbids them; only $8$ are informative. At $\tau{=}50\%$, $7$ of those $8$
forbid. Among the $22$ permitting cells the incumbent's own confidence reaches the target in $10$, but only
$15$ of them are measurable at all, because in the other $7$ the confidence signal is too tie-collapsed to
resolve a $10\%$ coverage. We report that denominator rather than scoring an unmeasurable cell as a failure.

\textbf{The forbidding cells, read at realised coverage.} A cell's ceiling is computed at $c{=}10\%$, but a
tie-collapsed signal answers fewer cases than that, and a smaller coverage raises the ceiling. Each cell must
therefore be read at the coverage its own signal actually resolves: $4$ of the $7$ forbidding cells are
\emph{not} forbidden there, and in $2$ more the signal resolves no coverage at all, so exactly \emph{one} of
the seven is a measured forbidding cell. That reading answers a different question from the one the criterion poses,
since a smaller coverage raises the ceiling, so it is a diagnostic about the signals' resolution, not a rescue of
the cells. The ceiling itself cannot be violated: $\mathrm{sel\text{-}acc}(c)\leq\min(1,p/c)$ is a counting
identity, so no cell could contradict it and none does. What is measured rather than entailed is the single
forbidding cell whose signal actually operates at $c{=}10\%$.

\textbf{Interpretation of the feasibility audit.} It shows the criterion is not vacuous: it forbids where base rates
have collapsed and permits where they have not, along a boundary that matches the rest of the paper, every
forbidding cell is on the ultra-rare stratum of the rare-disease task, while the entity-linking
cells have base rates of $65$--$84\%$ and are never forbidden. It does \emph{not} show seven independent
rescues: the external arm in those seven cells is the same retriever margin on the same ultra-rare cases
against seven different incumbents, so behind them lie only $2$ distinct external measurements.

\subsection{Confidence Calibration and Risk--Coverage}
\label{app:calib}
We score three confidence signals (verbalized-overall, verbalized-top1, mean token log-probability) by
AUROC-of-correctness and ECE per bin, with bootstrap $95\%$ CIs. (Correctness here is nearest-name cosine
$\geq\!0.90$; Tables~\ref{tab:collapse},~\ref{tab:method} use exact OMIM-id and the two agree qualitatively.)
Verbalized confidence's tail discrimination splits by scale: the \emph{smallest} models are near chance ($7$B
ultra-rare own-answer AUROC $0.571$, wide CI, $0.604$ on the full set; Llama-8B $0.509$), whereas the two larger
models rank their own tail answers \emph{well} ($14$B $0.777$, $32$B $0.905$). Good ranking does \emph{not} rescue
them: by the base-accuracy ceiling (\S\ref{sec:calib}) even Qwen-32B's near-oracle discrimination leaves its
most-confident $10\%$ far below a safe accuracy (Table~\ref{tab:method}). Read as a coverage, the same bound
says a small model can answer at most $p/\tau$ of the tail at accuracy $\tau$, which at $\tau{=}50\%$ is
$1.5$--$9.2\%$ of it, so raising coverage to $c{=}25\%$ is missed by more than $2.7\times$. This is because so few tail answers are correct
to begin with, so we do \emph{not} read the collapse as a discrimination failure. ECE worsens sharply on
the tail ($32$B less-rare $0.42\!\to\!$ ultra-rare $0.73$, corrected bins), and the $32$B mean-logprob AUROC
\emph{inverts} ($0.79\!\to\!0.36$, CI below $0.5$), so likelihood-based abstention would prefer wrong answers.
\textbf{Comparison across feasibility regimes.} If $p\!\ge\!\tau c$ is doing work rather than
decorating an identity, the external gate's advantage should decline as the incumbent's base accuracy rises. It
does. Table~\ref{tab:regimes} reads it at matched $10\%$ coverage, with the LLM's confidence given its
\emph{most favourable} tie-breaking.

\begin{table*}[t]
\centering\footnotesize\setlength{\tabcolsep}{6pt}
\begin{tabular}{@{}lccccc@{}}
\toprule
configuration & $p$ (tail base) & oracle ceiling & feasible? & margin@$10\%$ & advantage \\
\midrule
Yi-1.5-9B        & $0.7\%$  & $7\%$   & no  & $82.8\%$ & $+79.6$ \\
InternLM2.5-7B   & $1.4\%$  & $14\%$  & no  & $81.7\%$ & $+76.3$ \\
Mistral-7B-v0.3  & $1.6\%$  & $16\%$  & no  & $77.4\%$ & $+62.4$ \\
Qwen2.5-VL-7B    & $1.7\%$  & $17\%$  & no  & $80.6\%$ & $+66.7$ \\
Qwen2.5-14B      & $4.2\%$  & $42\%$  & no  & $74.2\%$ & $+46.2$ \\
Qwen2.5-VL-32B   & $4.3\%$  & $43\%$  & no  & $75.6\%$ & $+51.1$ \\
Llama-3.1-8B     & $4.6\%$  & $46\%$  & no  & $79.6\%$ & $+66.7$ \\
\midrule
DeepSeek-V4-Flash    & $15.8\%$ & $100\%$ & \textbf{yes} & $77.4\%$ & $+8.6$ \\
V4-Pro (reason.\ off) & $18.6\%$ & $100\%$ & \textbf{yes} & $76.3\%$ & $\mathbf{+0.0}$ \\
V4-Pro (reason.\ on)  & $22.4\%$ & $100\%$ & \textbf{yes} & $76.2\%$ & $\mathbf{-2.4}$ \\
\bottomrule
\end{tabular}
\caption{The external gate's advantage against the incumbent's own confidence, at matched $10\%$ coverage and
with the LLM's confidence given its most favourable tie-breaking. \emph{feasible?} is whether
$p\geq\tau c$ holds at $\tau{=}50\%$. The advantage collapses exactly where the ceiling stops binding.}
\label{tab:regimes}
\end{table*}

Mean advantage $+64.1$pp in the seven ceiling-bound configurations and $+2.1$pp in the three where it has lifted.

\textbf{Do not read a dose-response off this table.} The margin band is nearly flat ($74$--$83\%$), so the
advantage is essentially a constant minus the confidence band, and that band is bounded by $\min(1,p/c)$, a
deterministic function of $p$. Substituting \emph{any} quality of confidence ranker reproduces the observed
$r{=}{-}0.970$ to within $0.03$, and residualising on the ceiling leaves $-0.15$. The apparent trend is the
ceiling identity restated and we claim nothing from it.

\textbf{A lifted ceiling is not sufficient.} Sweeping \emph{coverage} on a fixed model moves the ceiling with
$p$ held exactly constant, which separates the ceiling from model quality. Llama-3.1-8B ($p{=}4.6\%$) has a
ceiling of $1.0$ at $c{=}2$--$3\%$ and its confidence band there is $10.7$--$11.1\%$ against the margin's
$77.8$--$82.1\%$: an advantage of $+67$ to $+71$pp in exactly the regime where the criterion no longer binds. The
three frontier configurations, where the two gates do converge, are one base model at three settings from one
vendor, and their per-configuration advantages all straddle zero. We therefore claim $p\!\ge\!\tau c$ as a
necessary condition, failing it forecloses every confidence policy, and claim nothing from passing it.

Two further limits on this table. The configurations cluster at $p\!\le\!4.6\%$ and $p\!\ge\!15.8\%$, so
$\tau c{=}5\%$ falls in an empty gap and the boundary's \emph{location} is not identified by our data. And the
confidence bands use the tie-breaking most favourable to the LLM; under random tie-breaking they fall further,
which widens the margin's advantage. It does \emph{not} leave everything unchanged: the go/no-go table's
permitting cells depend on it, falling from three of three to one of three (\S\ref{sec:calib}).

\textbf{Deployment risk--coverage.} none of these AUROCs buys a safe operating point. Gating
each deployable system by its own signal, a retriever-first system answers the top $10\%$ of ultra-rare cases at
$74$--$83\%$ accuracy, versus $2$--$18\%$ for the bare LLM under its own confidence (base Recall@1 $\leq\!5\%$).
Both numbers are levels, and the level is the part of this result that is \emph{not} invariant over the leakage
bracket: at the strict endpoint the retriever-first band falls to $12.3\%$, below Qwen-14B's observed $18\%$.
What survives the bracket is the \emph{lift} over each system's own base rate ($2.4$--$3.3\times$ for the
margin). A lift is what our claim rests on and it is not an operating point: siting a threshold needs the
level, which the bracket does not fix. The ordering of levels holds under leave-source-out and not under the
strict control.

\textbf{Certifying the operating point.} Our threshold is chosen, not certified, and this appendix records what
certifying it would take. Split-conformal selective prediction calibrates a margin threshold on a held-out half
so that empirical error on the answered part is at most $\alpha$, then reports the coverage achieved on the
test half; the guarantee is distribution-free and finite-sample, so it holds whatever the absolute accuracy is,
which is why the identification bracket of App.~\ref{app:decon} changes how many patients can be served at a
given promise rather than breaking the promise. At $\alpha{=}20\%$ and $\delta{=}0.10$ the calibration half
needs ${\approx}2{,}660$ answered cases, which at $10\%$ coverage is ${\approx}26{,}600$ patients; $\alpha{=}30\%$
needs ${\approx}50$ and $\alpha{=}40\%$ ${\approx}20$. The binding problem is not sample size but dependence.
Correctness is clustered within disease over $207$ diagnoses. A one-way random-effects ANOVA on those cases
gives ICC $0.553$ and a design effect of $2.91$, an effective $n$ of $322$; a second estimator archived with
the audit gives $0.704$ and $268$. We quote the range, $268$--$322$ independent cases, and a cluster-aware certifier evaluated on \emph{held-out
diseases} breaches its nominal risk on $11$--$44\%$ of splits against $13$--$23\%$ under i.i.d.\ case splits.
A deploying site routinely meets diagnoses absent from any calibration set it could have assembled, so the
held-out-disease split is the regime that matters.

\subsection{Recall at $k$}
\label{app:topk}
Every accuracy number in this paper is top-1, and the feasibility criterion is a top-1 argument, but the
systems emit a ranked top-5
differential and a clinician reads one. The two are not interchangeable, so we measure the gap rather than
concede it. Each of the five generated names is linked to OMIM by the same SapBERT nearest-candidate rule used
for top-1, so a name is scored at rank $5$ exactly as it is scored at rank $1$. One caveat belongs here rather
than in a footnote: BioMistral-7B emits a mean of $1.15$ names per case and exactly one in $90\%$ of them, so its
``$k{=}5$'' column is $k{=}1$ by output-format failure. Its ceiling is below $\tau$ either way, so the $4$-of-$8$
count does not change, but one of those four is a degenerate cell and the table should not be read as though it
were not.

\begin{table*}[t]
\centering\footnotesize
\setlength{\tabcolsep}{6pt}
\begin{tabular}{@{}lcccc@{}}
\toprule
 & less-rare & ultra-rare & ceiling & forbids \\
 & \multicolumn{2}{c}{\footnotesize R@1 / R@5} & \footnotesize @1 / @5 & \footnotesize at $k{=}5$ \\
\midrule
Qwen2.5-VL-7B         & $36.5\%$/$42.8\%$ & $1.7\%$/$3.8\%$ & $0.17$/$0.38$ & \textbf{yes} \\
Qwen2.5-14B           & $40.5\%$/$47.3\%$ & $4.2\%$/$8.1\%$ & $0.42$/$0.81$ & no \\
Qwen2.5-VL-32B        & $40.6\%$/$43.4\%$ & $4.3\%$/$7.3\%$ & $0.43$/$0.73$ & no \\
Llama-3.1-8B          & $39.2\%$/$43.2\%$ & $4.6\%$/$8.4\%$ & $0.46$/$0.84$ & no \\
Yi-1.5-9B             & $7.2\%$/$27.0\%$ & $0.7\%$/$2.7\%$ & $0.07$/$0.27$ & \textbf{yes} \\
InternLM2.5-7B        & $32.9\%$/$41.0\%$ & $1.4\%$/$2.2\%$ & $0.14$/$0.22$ & \textbf{yes} \\
Mistral-7B-v0.3       & $40.1\%$/$44.1\%$ & $1.6\%$/$6.7\%$ & $0.16$/$0.67$ & no \\
BioMistral-7B         & $36.0\%$/$36.5\%$ & $2.4\%$/$2.6\%$ & $0.24$/$0.26$ & \textbf{yes} \\
\midrule
DeepSeek-V4-Flash     & $40.5\%$/$49.1\%$ & $15.8\%$/$20.6\%$ & $1.00$/$1.00$ & no \\
DeepSeek-V4-Pro       & $48.2\%$/$50.9\%$ & $18.6\%$/$22.9\%$ & $1.00$/$1.00$ & no \\
V4-Pro (reason.\ on)  & $49.5\%$/$53.7\%$ & $22.4\%$/$26.6\%$ & $1.00$/$1.00$ & no \\
\bottomrule
\end{tabular}
\caption{Recall at the $k$ the deliverable is defined at, and the feasibility ceiling $\min(1,p/c)$ recomputed
from the base rate at that $k$. \emph{Forbids} marks the cells where the ceiling still lies below
$\tau{=}0.5$ at $c{=}10\%$, i.e.\ where the ceiling's arithmetic still applies. It does so for all $8$ small models at
$k{=}1$ but only $4$ at $k{=}5$.}
\label{tab:topk}
\end{table*}

Two things follow, and they cut in opposite directions. The collapse itself is robust to $k$: the small models'
ultra-rare Recall@5 is $2$--$8\%$, against a less-rare anchor that stays near $40\%$, so the collapse is not
a top-1 artifact. The infeasibility verdict is more fragile. The ceiling forbids the $\tau{=}50\%$, $c{=}10\%$ operating point for all $8$
small models at $k{=}1$ but for only $4$ at $k{=}5$: Qwen2.5-14B, Qwen2.5-VL-32B, Llama-3.1-8B and
Mistral-7B-v0.3 pass out of the forbidden region once success means gold-in-top-5. For those four the question
stops being arithmetic and becomes empirical, we have not shown that any confidence policy on them
\emph{does} reach the operating point, only that the ceiling no longer rules it out. The recommendation itself
is unaffected: the gate's advantage at $k{=}5$ is measured separately (App.~\ref{app:whichcases}) and survives.

\section{Phenotype-Ranker Analyses}

Whether a deployed tool reproduces the margin result, whether it survives incomplete phenotyping, and which patients the gate ends up selecting.

\subsection{Exomiser in Phenotype-Only Mode}
\label{app:exomiser}
\begin{figure}[t]
\centering\includegraphics[width=\columnwidth]{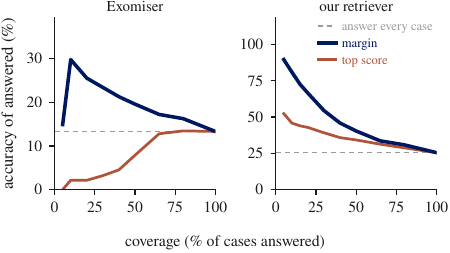}
\caption{Selective accuracy at different coverage levels for phenotype-only Exomiser (left) and our phenotype
retriever (right), selecting by the top score or the top-two margin; the horizontal line is accuracy when all
cases are answered. On Exomiser the top score never rises meaningfully above answering everything, so no
coverage makes it a usable gate.}
\label{fig:rc}
\end{figure}
We do not apply leave-source-out to Exomiser, so its recall belongs beside our undecontaminated
figures rather than beside the retriever's headline.
To anchor our reimplemented retriever against a tool clinicians actually run, we execute Exomiser 14.0.0
\citep{smedley2015exomiser} (2402 data release) in \emph{phenotype-only} mode, the mode used by
\citet{reese2026benchmarking}, on the identical $N{=}2000$ sample: each case's HPO terms as a phenopacket with no
variants (a shared empty hg38 VCF satisfies the CLI; the \texttt{phenotype-only} preset scores no variants).
Exomiser ranks genes by hiPHIVE phenotype similarity; we score \textbf{gene-level} Recall@1, did it rank the
case's causative gene (from the phenopacket) first, which for these Mendelian cases proxies disease-level but is
Exomiser's native output and \emph{not directly comparable} to the disease-level LLM/retriever numbers.

\textbf{Accuracy.} Gene-R@1 is $21.2\%$ (less-rare) / $13.3\%$ (ultra-rare) / $6.0\%$ (unknown); R@5
$42$/$28$/$13\%$. This is comparable to our retriever and $2.9$--$17.9\times$ the bare LLM's tail across the small models, but \emph{lower} than
the $35.5\%$ \citet{reese2026benchmarking} report, expected, given the gene-level metric (the causative
gene must rank first among ${\sim}20{,}000$, a different candidate space), the phenopacket-store case mix, and phenotype-only prioritization
without variant evidence. It confirms the paradigm-level point (phenotype tools $\gg$ bare LLM on the tail)
without a horse race, and the modest absolute ceiling ($13$--$24\%$ for both tools) is exactly why the
deployable contribution is triage, not accuracy.

\textbf{Margin-based triage with Exomiser.} Exomiser's top-1 combined \emph{score} cannot define a
$10\%$ operating point on this tail at all: \emph{no} case is strictly above the cut, and the whole decile is
drawn by sort order from a single $228$-case tie block, so the $2.2\%$ it appears to deliver is an artifact of
that ordering and not a gradient; we score ties by their expected contribution rather than by an arbitrary
break \citep{mcsherry2008ties}, and report the range the break spans where it matters. But ranking by the top1$-$top2 \emph{margin}, the top-$10\%$ reach
$29.0\%$ (ultra-rare) / $40.9\%$ (less-rare), well above base. So the margin-triage property of \S\ref{sec:method}
is not an artifact of our reimplementation: a real deployed tool's own margin also yields a usable operating
point on the tail. Two alternative readings of that result do not survive. It is not the tie block: on the
$711$ cases whose phenotype score has not saturated, the margin's decile still beats the raw score's by
$+28.5$pp, disease-clustered $[+15.4,+42.3]$pp, though the raw decile is only nominally below its own base
rate there, so we claim of it that it buys nothing rather than that it hurts. Nor is the raw score a straw
target: over the whole tail it gates at AUROC $0.383$, $0.507$ once the block is removed. The empirical
$p$-value Exomiser ships, a principled normalization already used in clinical filtering and the statistic a
deployment reaches for first, does no better: AUROC $0.415$, and as a triage signal its most confident $10\%$
delivers $7.5\%$, below the $13.3\%$ base rate. That is the ordering App.~\ref{app:margin_component}
predicts, since a normalization derived from the input rather than from the competitors in contention does
not recover the gate. Usable, not safe: $29.0\%$ is far below the $\tau{=}50\%$ this paper uses as its worked
example, and it is the margin, not the raw score, that yields it.

\subsection{Robustness to Phenotype Incompleteness}
\label{app:robust}
The benchmark's HPO terms are curated \emph{post-diagnosis}, so they are more complete and precise than a
prospective work-up, which could inflate the phenotype-only retriever. To test how far our conclusions survive
incomplete phenotyping, we drop a fraction $\rho$ of each case's \emph{present} HPO terms, at random, and, as an
\emph{adversarial worst case} for a phenotype-overlap retriever, \emph{highest information-content first} (dropping
the rare, distinctive findings it most relies on; the opposite regime, missing common low-IC terms, would hurt
it less), and re-run \emph{both} predictors on the \emph{identical} degraded
input: the retriever is re-scored, and the frontier DeepSeek-V4-Pro is re-queried on the degraded prompts (same
seed-0 $N{=}2000$ sample). At $\rho{=}0$ the harness reproduces the retriever's clean tail number on this sample
($25.6\%$; $25.6\%$ on the full decontaminated set).

\begin{table*}[t]\centering\small
\begin{tabular}{@{}lccc@{}}
\toprule
ultra-rare Recall@1 (highest-IC-first dropout, $N{=}2000$) & clean & drop $30\%$ & drop $50\%$ \\
\midrule
IC-overlap retriever                & $25.6\%$ & $\mathbf{16.6\%}$ & $\mathbf{13.9\%}$ \\
frontier LLM, reasoning \emph{off}  & $18.6\%$ & $5.6\%$ & $4.5\%$ \\
\bottomrule
\end{tabular}
\caption{Ultra-rare Recall@1 under simulated incomplete phenotyping, dropping the highest-information
findings first. Both predictors are re-run on the \emph{identical} degraded input; the LLM arm is
DeepSeek-V4-Pro with reasoning off, held per-case.}
\label{tab:robust}
\end{table*}

\noindent \textbf{Effects on the main conclusions.} \emph{The tail ranking is robust}: ``no LLM surpasses
the retriever on the tail'' survives realistic incompleteness, because the frontier LLM degrades at least as fast
on the tail (decisively for reasoning-off; a partial $n{=}127$ check for reasoning-on), so the predictor-selection
recommendation stands, and the anticipated failure (the LLM's
parametric priors letting it overtake the retriever under noisy input) does \emph{not} occur. The
deployment-relevant comparison survives too: under the same dropout the retriever's margin band still triages the
top-$10\%$ far above the frontier LLM's own confidence band ($62.4\%$ vs.\ $21.5\%$ at $\rho{=}0.30$, $39.8\%$ vs.\
$12.9\%$ at $\rho{=}0.50$; the LLM band on the looser cosine-correctness of App.~\ref{app:calib}, generous to the
LLM), because the LLM's cratered tail base rate re-imposes the ceiling on its own confidence. \emph{The absolute
operating point is not robust}: the retriever's tail accuracy falls, and its margin-triage safeguard weakens with
it (the \emph{retriever's own} top-$10\%$ accuracy, $82.8\%\!\to\!62.4\%\!\to\!39.8\%$ at $\rho{=}0,0.30,0.50$;
Table~\ref{tab:method}'s $74$--$81\%$ is the delivered-hybrid figure), so the ``safe''
band is itself conditional on phenotyping quality, a caveat we state where the claim is made
(\S\ref{sec:method}, Limitations). The base-accuracy ceiling and the small-model negative are \emph{un}affected:
degrading phenotypes only lowers the tail base rate, tightening the ceiling.

\subsection{Characteristics of Answered Cases}
\label{app:whichcases}
A margin gate would be of little clinical interest if it fired only on cases carrying one unmistakable,
disease-defining finding, the cases a specialist resolves without help. We test this directly on the full
ultra-rare tail ($n{=}4{,}780$): rank cases by the decontaminated retriever's top1$-$top2 margin, take the answered
top-$10\%$ ($n{=}478$), and compare them with the deferred remainder.

\begin{table*}[t]\centering\small
\begin{tabular}{@{}lcc@{}}
\toprule
ultra-rare tail ($n{=}4{,}780$) & answered (top-$10\%$ margin) & deferred \\
\midrule
Recall@1                                   & $\mathbf{81.0\%}$ & $19.2\%$ \\
\# matched gold findings                   & $\mathbf{65.4}$   & $32.9$ \\
\# HPO terms in the case                   & $15.3$            & $9.0$ \\
share of evidence from the single strongest finding & $5.3\%$  & $10.4\%$ \\
\quad, \,relative to the uniform $1/n$ share        & $2.7\times$ & $2.5\times$ \\
max IC of a matched finding                & $7.08$            & $4.43$ \\
cases per disease in the corpus (frequency proxy) & $\mathbf{51.9}$ & $111.0$ \\
\bottomrule
\end{tabular}
\caption{Characteristics of the cases the margin gate answers against those it defers. \emph{share of
evidence} is the single strongest matched finding's contribution, given raw and relative to the uniform
$1/n$ share; the frequency proxy is that disease's case count in this corpus.}
\label{tab:whichcases}
\end{table*}

The answered cases are not ``one give-away finding'' cases. They are more richly phenotyped and are decided
by roughly twice as many converging findings, and the single strongest finding carries the \emph{same}
relative share of the evidence in both groups ($2.7\times$ against $2.5\times$ the uniform share), so the raw
difference ($5.3\%$ against $10.4\%$) follows mechanically from having more matched findings rather than from
greater peakiness. What distinguishes the answered set is the breadth of converging evidence. The gate also
fires on \emph{less} frequently represented diseases ($51.9$ against $111.0$ cases per disease), so it is not
concentrating on the corpus's famous entities. These are properties of case \emph{structure}, a proxy for
where multi-finding integration is hard; we do not measure clinician performance.

\textbf{Per-disease.} The answered decile spans $94$ diseases (top disease $13.2\%$ of it), so we also report it
macro-averaged: $81.0\%$ micro becomes $64.6\%$ per-disease (disease-clustered $95\%$ CI $[54.7,73.5]$), still
$2.6\times$ its own base. The gate's advantage is not an artifact of a few well-represented diseases.

\section{Retriever--LLM Hybrids and Trained Fusion}

\label{app:method}
\label{app:fusion}
All numbers use the decontaminated retriever (App.~\ref{app:decon}) and the $N{=}2000$ per-model sample, with
$2000$-resample bootstrap intervals.

\textbf{Hybrid-gate performance.} We gate the delivered answer of a hybrid predictor (keep the LLM top-1 when
the retriever ranks it in its top-5, else output the retriever's top-1) by each signal. On every tail bin the
retriever's own margin dominates. Over Table~\ref{tab:method}'s four small configurations the ultra-rare
AUROC is $0.74$--$0.77$ for the margin against $0.54$--$0.56$ for LLM confidence and $0.48$--$0.53$ for
agreement, and $0.73$--$0.74$ against $0.45$--$0.52$ and $0.49$--$0.58$ on unknown-prevalence cases. Only on
less-rare disease does agreement lead ($0.79$--$0.85$). Relaxing the
operating point from $10\%$ to $25\%$ coverage shows the usual selective-prediction trade: on the ultra-rare tail the margin band
falls from $74$--$83\%$ to $50$--$57\%$ across all ten gated configurations
(Table~\ref{tab:method}'s four span $74$--$81\%$), so by a quarter coverage the answered set already carries
$43$--$50\%$ errors. This is why we describe the gate as a triage safeguard rather than a cure.

\textbf{Agreement as an LLM reliability signal.} Whether retriever--LLM agreement predicts the LLM's
\emph{own} top-1 correctness is a separate question from whether it should gate the delivered answer, and it is
worth separating because agreement could be a restatement of case difficulty. It is not:
Table~\ref{tab:reliability} shows agreement predicting LLM correctness far above the ``the retriever itself
solved the case'' difficulty proxy on all $18$ bins, including the ultra-rare tail ($32$B tail $0.919$
against $0.787$). Agreement is therefore a genuine, non-circular reliability signal, but once a system already
defers to the retriever, the LLM's contribution to the \emph{delivered} tail answer vanishes.

\begin{table*}[t]
\centering\small
\begin{tabular}{@{}lcccccc@{}}
\toprule
 & \multicolumn{3}{c}{agreement $\to$ LLM-correct} & \multicolumn{3}{c}{trivial ``retriever solved it'' $\to$ LLM-correct} \\
\cmidrule(lr){2-4}\cmidrule(lr){5-7}
Model & less & ultra & unknown & less & ultra & unknown \\
\midrule
Qwen2.5-VL-7B  & \textbf{0.990} & \textbf{0.887} & \textbf{0.840} & 0.579 & 0.688 & 0.683 \\
Qwen2.5-14B    & \textbf{0.917} & \textbf{0.852} & \textbf{0.820} & 0.544 & 0.674 & 0.640 \\
Qwen2.5-VL-32B & \textbf{0.918} & \textbf{0.919} & \textbf{0.855} & 0.557 & 0.787 & 0.703 \\
Llama-3.1-8B   & \textbf{0.890} & \textbf{0.904} & \textbf{0.814} & 0.514 & 0.732 & 0.629 \\
DeepSeek-V4-Flash & \textbf{0.869} & \textbf{0.762} & \textbf{0.796} & 0.572 & 0.529 & 0.713 \\
DeepSeek-V4-Pro   & \textbf{0.869} & \textbf{0.764} & \textbf{0.776} & 0.591 & 0.575 & 0.673 \\
\bottomrule
\end{tabular}
\caption{Predicting the LLM's \emph{own} top-1 correctness (AUROC, decontaminated retriever; six models incl.\ the
two frontier ones). Retriever--LLM agreement (left) is far above the trivial ``did the retriever's own top-1 match
gold'' difficulty proxy (right) on every bin, so agreement carries reliability information beyond case difficulty
(non-circular). This is the \emph{estimator} view; for \emph{deployment} the retriever's own margin is the better
gate (Table~\ref{tab:method}).}
\label{tab:reliability}
\end{table*}

\textbf{Trained fusion.} Our deployable rules are heuristics, so the negative in \S\ref{sec:method} could be an
artifact of choosing the wrong rule, especially since an oracle over $\{$LLM, retriever$\}$ leaves real headroom
on the tail. We therefore fit the obvious harvester: a per-case ``trust the LLM?'' classifier on the deployable
signals (LLM verbalised confidence, LLM mean token log-probability where available, the retriever's
top1$-$top2 margin, and the retriever's rank of the LLM's top-1), delivering the LLM's top-1 when it fires and
the retriever's otherwise, scored on held-out folds ($5$-fold stratified CV, so no case is decided by a
classifier that saw it). We report the better of logistic regression and gradient boosting.
The complementarity is real but not harvestable. At the frontier an oracle over the two predictors would
reach $39.2\%$ against the retriever's $25.4\%$, yet the trained fusion captures $+0.2$pp and no configuration
yields a significant gain; for the small models the classifier correctly learns never to trust the LLM, so
the fusion reduces exactly to the retriever. The estimate is generous to the fusion, being fitted and
evaluated on the same tail distribution with no deployment shift, and it still does not help. That closes the
objection that our negative is specific to the two heuristic rules of \S\ref{sec:method}.

\begin{table*}[t]\centering\small
\begin{tabular}{@{}lcccccc@{}}
\toprule
ultra-rare tail & LLM & retriever & heuristic hybrid & \textbf{trained fusion} & oracle & $\Delta$ vs.\ retr.\ ($p$) \\
\midrule
Qwen2.5-VL-7B     & $1.7\%$  & $25.6\%$ & $24.9\%$ & $25.6\%$ & $26.3\%$ & $+0.0$ ($1.00$) \\
Qwen2.5-14B       & $4.2\%$  & $25.6\%$ & $23.5\%$ & $25.6\%$ & $27.4\%$ & $+0.0$ ($1.00$) \\
Qwen2.5-VL-32B    & $4.3\%$  & $24.6\%$ & $23.0\%$ & $24.6\%$ & $25.6\%$ & $+0.0$ ($1.00$) \\
Llama-3.1-8B      & $4.6\%$  & $25.6\%$ & $24.1\%$ & $25.6\%$ & $26.8\%$ & $+0.0$ ($1.00$) \\
Yi-1.5-9B         & $0.7\%$  & $25.6\%$ & $24.9\%$ & $25.6\%$ & $26.1\%$ & $+0.0$ ($1.00$) \\
InternLM2.5-7B    & $1.4\%$  & $25.6\%$ & $24.7\%$ & $25.6\%$ & $26.2\%$ & $+0.0$ ($1.00$) \\
Mistral-7B-v0.3   & $1.6\%$  & $25.6\%$ & $24.0\%$ & $25.6\%$ & $26.3\%$ & $+0.0$ ($1.00$) \\
DeepSeek-V4-Flash & $15.8\%$ & $25.6\%$ & $25.2\%$ & $25.8\%$ & $36.5\%$ & $+0.2$ ($0.50$) \\
DeepSeek-V4-Pro (off) & $18.6\%$ & $25.6\%$ & $25.6\%$ & $25.7\%$ & $37.2\%$ & $+0.1$ ($1.00$) \\
DeepSeek-V4-Pro (on)  & $22.4\%$ & $25.4\%$ & $26.7\%$ & $25.7\%$ & $\mathbf{39.2\%}$ & $+0.2$ ($0.77$) \\
\bottomrule
\end{tabular}
\caption{Retriever alone against two ways of combining it with the LLM, on the ultra-rare tail. $\Delta$ is
against retriever-alone by exact paired McNemar; the oracle column is the ceiling a perfect per-case choice
between the two predictors would reach.}
\label{tab:fusion}
\end{table*}

\section{Audit of SapBERT-to-OMIM Linking}

\label{app:linkaudit}

Every LLM number here passes through one instrument: a free-text disease name linked to OMIM by SapBERT
nearest-candidate cosine. If that instrument degraded on ultra-rare names, long, eponymous, numbered, synonym
rich, it would \emph{manufacture} the collapse. Two tests bound this, and neither needs annotation.

\textbf{Held-out surface form.} The test must give the linker a form it has not already been handed, or it
measures nothing: feeding back a candidate's own name scores $100\%$ by construction, since the query embedding
\emph{is} the candidate embedding. We use an Orphanet synonym that is not the OMIM title and that exactly one
entry claims. On $6{,}776$ such probes the ultra-rare tail resolves at $39.6\%$, against $49.9\%$ on less-rare
disease ($n{=}680$; the Wilson intervals do not overlap). Every eligible probe is used and the hash seed is pinned, since
sampling probes from an unordered set moves the result by up to $3$pp. The linker does degrade on the
tail, so we do not rest the argument on it. The weight falls on the next test, whose alias arm is string matching
and does not pass through SapBERT at all.

\textbf{Generous re-scoring.} Credit a model if \emph{any} of its five names links to the gold \emph{or} matches
an alias of it, the gold's OMIM title or an Orphanet synonym, after case, accent and punctuation
normalisation. Orphanet attaches a group's name to every OMIM entry it references, so any normalised form
claimed by more than one entry is discarded ($2{,}895$ of $21{,}161$; $91$ diseases lose every unambiguous
form). Without that guard the rule silently scores at group level: stripping
``syndrome''/``type'' and sorting tokens leaves one form covering $108$ unrelated neurodevelopmental entries, and
\emph{every} extra credit the retriever received came from a different OMIM entry rather than an alias of the
gold.

Under the corrected rule the small models reach $2$--$8\%$ on the tail, against $\leq\!5\%$ strict, the alias
arm adds almost nothing beyond top-5, i.e.\ the linker was already seeing what it could see. Applied to the retriever the same rule adds exactly $0.0\%$, but that is an identity, not a
measurement: after the guard no surviving form is owned by more than one entry, so for an identifier-predicting
system the generous rule reduces to the strict one. The rule is therefore \emph{one-sided}: it can credit a
name-generating system and cannot credit an identifier-predicting one. That asymmetry is why the arithmetic
below, not this $0.0\%$, is what closes the objection.

What closes the objection is arithmetic rather than either test on its own. Take the best small model's strict
tail Recall@1 ($4.6\%$) and inflate it by the linker's own held-out resolution rate on that stratum, as though
\emph{every} miss the linker makes were a correct answer thrown away: $11.6\%$. That is a deliberately
over-generous correction and it still leaves the best small model less than half the retriever's $25.6\%$. What
none of this establishes is the linker's \emph{precision}, whether a link goes to the right disease rather than
a plausible neighbour. That needs a genetics-literate adjudicator and we do not report one; it is a different
quantity from the one the objection turns on.

\section{Cross-Task Transfer and Boundary Conditions}

\label{app:transfer}
This appendix carries the two decisions to standard ranking tasks, in the order the framework poses them:
which statistic reads top-1 correctness, which reads candidate presence, what conditions the margin needs
before either question is well posed, and whether the comparisons survive being read off the whole
risk--coverage curve rather than one operating point.

\subsection{Top-1 Correctness across Scorers}

Entity linking's standard NIL rule thresholds the top-1 score rather than the margin \citep{sevgili2022neuralel}, exactly the
quantity that fails for our prioritizers. On $12{,}741$ BC5CDR disease mentions linked to MEDIC with SapBERT
the two signals are \emph{not distinguishable}: at the $10\%$ operating point this paper uses throughout the
margin is nominally ahead ($92.9\%$ vs.\ $92.2\%$ overall; $96.6\%$ vs.\ $93.2\%$ on the rare-concept tail,
$n{=}591$) while by AUROC the raw cosine is nominally ahead ($0.822$ vs.\ $0.796$), and neither gap survives a
bootstrap clustered on the $3{,}201$ unique mention strings the corpus actually contains
($\Delta$AUROC $-0.027$, $95\%$ CI $[-0.075,+0.021]$; the sign is the margin minus the raw score,
so the point estimate favours the raw score and neither arm is separated). The tie is what the screen calls for rather than a
failure of it: $63.7\%$ of
mentions are exact string matches, so the raw cosine is saturated; on the $4{,}629$ mentions where it is not
(base $53.4\%$) the AUROC ordering \emph{reverses}, margin $0.770$ against raw $0.764$. That is a subgroup observation and not a screen, since \S\ref{sec:ranklimit} rules out screens of this kind, and
it is reported here because the split is diagnostic rather than predictive: where the score has spare
resolution the margin is ahead, and where it has saturated the comparison is decided by ties. The raw arm's
$92.2\%$ is in fact the most favourable of its tie-breaks, with $808$ mentions tied at the $10\%$ cut against
the margin's $49$, so the operating point here is as tie-sensitive as the one we criticise in Exomiser. The
reversal on the unsaturated subset is not itself significant. What we can say is that the recommendation is
established for \emph{unnormalized accumulation} scores, IC-weighted overlap and Exomiser's combined score, and
is untested for already-normalized similarities. The test the scope claim calls for holds the corpus and the candidates
fixed and changes only the scorer's normalization: replacing SapBERT cosine with BM25, an
unnormalized accumulation of IDF-weighted matches, makes the margin significantly better than the raw score by
AUROC ($0.793$ vs.\ $0.743$; bootstrap clustered on unique mention strings, $\Delta$ $95\%$ CI
$[+0.017,+0.082]$), with the same ordering on the rare-concept tail ($0.800$ vs.\ $0.763$); subtracting the
per-mention nuisance the score accumulates, its query IDF mass, is better still ($0.808$). The crossover is in
the predicted direction and was registered before the run. What it establishes is a \emph{ranking} result: at
the $10\%$ operating point the raw score is nominally ahead by $8.2$pp, but that gap is not resolved
($95\%$ CI $-0.6$--$+17.5$, straddling zero), so we claim the crossover for the ranking and nothing for the
operating point. A third arm runs a deployed late-interaction retriever. We score $300$ SciFact queries with the released
ColBERTv2 checkpoint under its own configuration ($128$-dimensional projection, queries padded to $32$
tokens with \texttt{[MASK]}, documents truncated to $180$ with punctuation zeroed, cosine MaxSim), so the
scoring function is the model's own and not a reimplementation. Its top score carries a per-query level ($C\!=\!0.374$), gates at AUROC $0.764$ against a base of $54.3\%$, and its own margin reaches $0.836$.

The SUM-versus-MEAN pair reported below runs on our own encoder rather than this one, and the reason is
structural. On our encoder the top score's scale tracks query length at $\rho\!=\!0.97$ and query-length
normalization recovers about half of what the margin recovers. The deployed model has no such dependence,
$\rho\!=\!+0.003$, because
ColBERT's query augmentation pads every query to the same $32$ tokens, so its sum runs over a constant number
of terms and dividing by query length divides by a constant. The two aggregations are then the same ranking and the
same score up to scale, so on the deployed model there is no SUM-versus-MEAN contrast to draw. On our own
encoder there is, and it is a controlled manipulation of normalization with the corpus, the candidates and the
encoder held fixed, one of the two controlled pairs \S\ref{sec:scope} uses to test $C$. The two arms make
different points: the pair isolates normalization, and the deployed run shows a large between-input share and a
large margin advantage in a system we did not build.
\subsection{Candidate Presence}
\textbf{Link or NIL: the same score read the other way.} Every BC5CDR mention above has a gold concept in
MEDIC, so those arms measure only whether the top-1 is correct. The decision a deployed linker also faces is
whether to link at all, and the decomposition of Eq.~\eqref{eq:scoredecomp} sends the two decisions to
different statistics: correctness is a question about $r_1-r_2$, whereas a mention whose gold concept is
absent is marked by the whole candidate set scoring low, which is $b(x)$, the term the margin cancels. We
therefore predicted, before running it, that the margin would \emph{lose} the NIL decision under both scorers,
and that the raw top score would be a good NIL detector only for the scorer whose $b(x)$ is small.

Following \citet{zhu2023nil}, we mask a random slice of MEDIC out of the candidate space, so mentions whose
gold concepts are all masked become genuinely unlinkable; masking $10/25/50\%$ of concepts makes $10.0/25.8/44.7\%$
of the $12{,}741$ mentions NIL. Both questions are then asked of one run, with bootstrap intervals clustered
on the mention string.

\begin{table}[t]
\centering\footnotesize\setlength{\tabcolsep}{3.5pt}
\begin{tabular}{@{}llcccc@{}}
\toprule
 & & raw & margin & gap@$10$ & level \\
\midrule
\multicolumn{6}{@{}l}{\emph{Q\textsubscript{NIL}: is the gold concept in the KB?}} \\
SapBERT & mask $10\%$ & $\mathbf{0.842}$ & $0.767$ & $0.825$ & -- \\
        & mask $25\%$ & $\mathbf{0.841}$ & $0.787$ & $0.832$ & -- \\
        & mask $50\%$ & $0.843$ & $0.797$ & $\mathbf{0.845}$ & -- \\
BM25    & mask $10\%$ & $0.673$ & $0.725$ & $0.636$ & $\mathbf{0.775}$ \\
        & mask $25\%$ & $0.657$ & $0.721$ & $0.645$ & $\mathbf{0.767}$ \\
        & mask $50\%$ & $0.702$ & $0.689$ & $0.657$ & $\mathbf{0.771}$ \\
\midrule
\multicolumn{6}{@{}l}{\emph{Q\textsubscript{correct}: among linkable, is the top-1 right?}} \\
SapBERT & mask $25\%$ & $\mathbf{0.869}$ & $0.831$ & $0.862$ & -- \\
BM25    & mask $25\%$ & $0.760$ & $0.808$ & $0.761$ & $\mathbf{0.850}$ \\
\bottomrule
\end{tabular}
\caption{Two abstention decisions on one corpus (AUROC). \emph{margin} and \emph{gap@$10$} are both zero-sum
contrasts over the candidates, $w=(1,-1,0,\dots)$ and $w=(1,-\tfrac19,\dots,-\tfrac19)$, so by
Eq.~\eqref{eq:contrast} each cancels $b(x)$ exactly. \emph{level} is instead the top score with a case-side
nuisance estimate removed, the query's own IDF mass, which only BM25 admits: SapBERT's cosine is bounded and
its raw score is already the level, hence the dashes.}
\label{tab:nil}
\end{table}

Table~\ref{tab:nil} reports it. Under SapBERT the raw cosine is the better NIL detector at every masking rate
and the margin is worse by a paired $-0.054$ ($95\%$ CI $[-0.088,-0.023]$ at $25\%$), as predicted. Under BM25
the raw score is a far weaker NIL detector, $0.657$--$0.702$ against the cosine's $0.841$--$0.843$, because its
level is dominated by the query's own IDF mass rather than by evidence about the knowledge base; subtracting
that mass beats the raw score by $+0.102/+0.109/+0.067$ across the three rates, every interval excluding zero.
BM25 is the one scorer here that carries both families on one run, so the two can be told apart rather than
argued about: alongside that case-side level sits the top score against the mean of the field, a zero-sum
contrast by Eq.~\eqref{eq:contrast}, and it is the \emph{worst} of the four statistics at every rate
($0.636$--$0.657$ against the level's $0.767$--$0.775$), falling behind even the raw score at $50\%$ masking
($-0.045$, $[-0.076,-0.015]$). Estimating the case level off the contenders does not approximate it; it
removes it.
The same protocol on MedMentions replicates the core and is reported here rather than only counted.
Disease-mention linking against UMLS on a PMID-disjoint split ($n{=}3{,}831$ evaluation mentions) gives a
knowledge base with $2.1$ surface forms per concept against MEDIC's seven, so the aliases the cosine relies on
are far thinner. Over the same $6$ (masking rate $\times$ scorer) cells the margin is the best
$Q_{\mathrm{NIL}}$ statistic in none and is resolvedly behind the raw score in $3$: at $10\%$ masking under
SapBERT the raw cosine reaches $0.902$ against the margin's $0.684$, a paired $-0.216$ ($[-0.320,-0.117]$
clustered on the mention string), three times the size of the same quantity on BC5CDR. What does not travel is
the BM25 half, and the thin alias base is why.

Our registered prediction had two clauses and one of them failed. The margin is never the best
$Q_{\mathrm{NIL}}$ statistic under either scorer, and it stays ahead of the raw score on $Q_{\mathrm{correct}}$
under BM25; that dissociation is what the decomposition predicts, and it holds in all six cells here and all
six on MedMentions. But we also predicted the margin would lose to the raw score under \emph{both} scorers,
and under BM25 it does not: it is ahead at two of the three masking rates. That clause was wrong, for a
reason the decomposition itself supplies: BM25's raw score is not a clean reading of $b(x)$, being dominated by
the query's own IDF mass, so beating it is not evidence about $b$ at all. Subtracting that mass restores the
predicted ordering at every rate. What the decomposition licenses is that whichever statistic reads $b(x)$
cleanly wins the presence question, not that the raw score is that statistic in a scorer where the two come
apart.

One artifact of the protocol works against the margin and we state it rather than argue it away: masking
removes competitors at random, which perturbs the runner-up and so adds noise to the margin specifically,
whereas a linkable mention's gold survives by construction and its top score does not move. The three
reported rates are nested prefixes of one permutation, so they are one draw and not three replicates, and we
do not read a trend across them. What we do read is that the margin is behind at all three rates and behind
under both scorers, and that the case-side level, which is perturbed by the same masking, is ahead.

\textbf{The same split on passage retrieval.} A reviewer may read the dissociation as a property of
short-mention linking rather than of ranking scores, so the masking protocol is transplanted unchanged to
passage retrieval. Masking a random slice of SciFact's $5{,}183$ abstracts makes a query unanswerable when
every one of its qrel-positive abstracts is gone; $10$/$25$/$50\%$ masking leaves $8.7$/$23.7$/$50.3\%$ of the
$300$ queries unanswerable. ColBERTv2 scores what remains. No case-side nuisance estimate is available for a
late-interaction score, so the third statistic here is the top score against the mean of the nine candidates
behind it, which Eq.~\eqref{eq:contrast} places in the same cancelling family as the margin rather than
alongside BM25's level.

\begin{table}[t]
\centering\footnotesize\setlength{\tabcolsep}{4pt}
\begin{tabular}{@{}llccc@{}}
\toprule
 & mask & raw & margin & gap@$10$ \\
\midrule
\multicolumn{5}{@{}l}{\emph{Q\textsubscript{exists}: is a relevant abstract still there?}} \\
 & $10\%$ & $\mathbf{0.696}$ & $0.647$ & $0.673$ \\
 & $25\%$ & $0.717$ & $0.680$ & $\mathbf{0.717}$ \\
 & $50\%$ & $0.724$ & $0.696$ & $\mathbf{0.738}$ \\
\midrule
\multicolumn{5}{@{}l}{\emph{Q\textsubscript{correct}: among answerable, is the top-1 right?}} \\
 & $10\%$ & $0.759$ & $\mathbf{0.835}$ & $0.814$ \\
 & $25\%$ & $0.753$ & $\mathbf{0.833}$ & $0.809$ \\
 & $50\%$ & $0.794$ & $0.836$ & $\mathbf{0.857}$ \\
\bottomrule
\end{tabular}
\caption{The dissociation on passage retrieval (AUROC, ColBERTv2 on SciFact). \emph{gap@$10$} is the top
score against the mean of the field, a zero-sum contrast like the margin. The margin is the \emph{worst} of
the three for whether a relevant abstract survives, at every masking rate, and beats the raw score on whether
the top-1 is right, at every masking rate.}
\label{tab:cnil}
\end{table}

Table~\ref{tab:cnil} reports it. The margin is the best $Q_{\mathrm{exists}}$ statistic at none of the three
rates, being in fact the worst of the three at all of them, and it beats the raw score on
$Q_{\mathrm{correct}}$ at all of them, by $+0.043$ to $+0.080$. The third column is the one to read
carefully. It was registered as a level, and Eq.~\eqref{eq:contrast} says it is not one: it cancels $b(x)$
exactly as the margin does, so the framework gives it no claim on the presence question. At the lightest
masking it duly loses to the raw score ($0.673$ against $0.696$); at $25$ and $50\%$ it passes, and what
passes there is not a recovered level but a field that has been thinned until its own flatness is
informative. BM25 is where the two readings separate, because only there is a genuine case-side estimate
available to compete: the same contrast is the worst of four at every rate while the case-side level is the
best. What the decomposition forbids is a contrast winning the presence question against a statistic that
reads $b(x)$, and no cell here or on the two linking corpora violates it.

\subsection{Boundary Conditions}
A top1$-$top2 margin measures competition between candidates, so it is informative only
where the top two are genuinely different answers. That is a condition on the candidate set and not on the task
label: SciFact is passage retrieval and the margin gains $+0.072$ AUROC there, and its top two abstracts are
genuine alternatives rather than paraphrases: they share $7.3\%$ of their content words against $3.0\%$ for
a random same-corpus pair, and $0\%$ of queries have a near-duplicate pair. Where the top two \emph{are}
near-duplicates, $r_1\!\approx\!r_2$ by construction and their difference is noise: sampled generation fails
the condition outright. Our own domain
satisfies the condition by construction, one gold OMIM disease among $8{,}553$, as do entity linking and
concept normalization against a fixed knowledge base. Before transferring the recommendation the question to
ask is therefore whether a system's runner-up is a rival answer or a paraphrase of the leader; where it is a
paraphrase, de-duplicating the candidate list restores the condition rather than defeating it.

The margin also needs a ranked candidate set scored by one scorer, and repeated
samples from a generator are not that. Drawing $k{=}10$ independent continuations per case from an
instruction-tuned $7$B model ($n{=}600$, base $17.3\%$), the margin between the best and second-best sample
carries no usable signal (AUROC $0.409$, disease-clustered CI $0.378$--$0.524$) while dividing the sequence
log-probability by its length reaches $0.872$. The reason is structural: in $87.7\%$ of cases the two
highest-scoring samples are the identical string, so there is no runner-up to measure a gap against. Nor is
the length divisor removing a nuisance there, the selected sample's length and its raw score are uncorrelated
($\rho\!=\!-0.03$), but longer answers happen to be likelier correct (length alone gives AUROC $0.766$), so
dividing by it injects a task-specific cue rather than cancelling a scale. The same operation is
nuisance-removal in one setting and signal-injection in the other, which is why the recommendation has to name
the structure it needs. The base-accuracy criterion carries no such restriction; it is the \emph{gate} that does.

\subsection{Whole-Curve Evaluation}
A comparison made at one operating point need
not survive at another, and \citet{traub2024flaws} propose the area under the generalized risk--coverage curve
(the mean over thresholds of $P(\text{fail and accept})$, which unlike the ordinary risk--coverage integral
does not condition on the accepted set) as the threshold-free alternative. It cannot serve as an independent
check of our comparisons. Writing $p$ for base accuracy, that area satisfies
$\mathrm{AUGRC}=p(1-p)\,(1-\mathrm{AUROC})+\tfrac12(1-p)^2$; the raw score and the margin are two readings of
one predictor's output and therefore share $p$ exactly, so at fixed $p$ the area is a decreasing affine
function of AUROC and the two metrics order the pair identically. We verified the identity on our own runs, to
within $10^{-5}$ of the empirical area. We therefore report AUROC beside the operating point and claim nothing
further from the threshold-free version. Our retriever on the ultra-rare tail: $0.317$ for the margin
against $0.343$ for the raw score, paired difference $[-0.037,-0.017]$ clustered on disease. Phenotype-only Exomiser: $0.407$
against $0.447$ on the ultra-rare stratum. BC5CDR with BM25: $0.102$ against $0.113$. BC5CDR with SapBERT,
where we report a tie: $0.061$ against $0.056$, paired $[-0.004,+0.012]$, which straddles zero. Every cell
that the operating point calls for the margin the whole curve also calls for the margin, and the one tie stays
a tie. The reverse does not hold everywhere: under BM25 the whole curve favours the margin while the $10\%$
operating point favours the raw score by $8.2$pp, an unresolved gap either way, which is why the entity-linking
claim in \S\ref{sec:transfer} is stated by AUROC and not at an operating point.

\section{Clinical Grading of Residual Errors}

\label{app:errgrade}

Exact-match accuracy treats every wrong answer alike. A clinician does not: a different genetic subtype of the
disorder the patient actually has sends them to the same panel, and a disease of another organ system sends them
elsewhere. We grade each wrong top-1 on three axes, none of which any system here is scored on or optimises for.
\textbf{Same-test}: the predicted and true disease share a causative gene, or map to the same Orphanet disorder.
\textbf{Same-class}: they share an ICD-10 category. \textbf{Unrelated}: measurable on at least one axis, related
on none. \textbf{Unmeasurable} is kept as its own bucket and never folded into ``unrelated'', $9.6\%$ of the
candidate space carries no gene, no Orphanet mapping and no ICD-10 code, and calling that unrelated would
manufacture harm.

\textbf{Decontamination applies to every candidate.} Leave-source-out drops terms sourced solely to the case's
own publication, and it is applied \emph{symmetrically}, to the gold profile and to every competitor. Stripping
only the gold would leave a same-gene sibling annotated from this very patient, routine, since one paper often
reports several allelic entries, holding evidence the gold had lost; such a candidate wins on curation
provenance and not on clinical adjacency. The leakage sweep of App.~\ref{app:decon} retains the gold-only rule,
so its bracket is on that rule and we do not claim otherwise. Per-bin Recall@1 moves by up to $0.9$pp between
runs and between float32 and float64, so we grade the gates and report no tail-accuracy delta.

Three design choices decide whether these numbers mean anything. \emph{(i)} ICD-10 is taken strictly: rare
syndromes pile into residual buckets. Q87 ``other specified congenital malformation syndromes'' and kin, which we
identify from the data as the categories carrying more than $150$ distinct disorders, so a shared category counts
only if it is neither a catch-all nor a ``.8/.9'' residual subdivision. The loose rule is computed alongside and
would move $137$ of the deferred ``unrelated'' errors into same-class; none of the answered ones. Strictness is
the conservative direction, since it makes \emph{unrelated} larger. \emph{(ii)} Phenotypic overlap is deliberately
excluded as a grading axis for the retriever: it ranks by IC-weighted phenotype overlap, so grading its errors
that way is circular. (For the record it goes the same way, $0.631$ answered vs.\ $0.520$ deferred; we do not use
it.) \emph{(iii)} The gene table ships with Exomiser's \texttt{2402} release and predates the corpus, which alone
put $237$ gold diseases in the unmeasurable bucket, detectable because their median OMIM identifier is visibly
newer. Phenopacket Store records each case's causative gene as its directory, so we merge those $293$ pairs in.
Coverage stays asymmetric, gold-side is near-complete, prediction-side is not, which biases \emph{against}
finding same-test, so the reported share is a lower bound.

Two worked pairs, so the grading can be checked by eye, and so the limit of it is visible. In one answered case
the true diagnosis is \emph{Greig cephalopolysyndactyly} and the retriever answers \emph{postaxial polydactyly,
types A1 and B}; exact match scores this zero, both are \emph{GLI3}, both are dominant, and the test the
suggestion triggers is the test that finds the answer. In another the truth is \emph{Robinow syndrome, autosomal
recessive} and the answer is \emph{Robinow syndrome, autosomal dominant 1}: the same Orphanet disorder, so a
panel still finds it, but the inheritance is inverted and the recurrence risk a family would be quoted is wrong.
The first pair is the case for the finding; the second is why it must be scoped.

\begin{table}[t]
\centering\small
\begin{tabular}{lrrr}
\toprule
 & \textbf{answered} & \textbf{deferred} & \textbf{LLM tail} \\
\midrule
wrong top-1 ($n$)      & $58$     & $1276$   & $923$ \\
\midrule
same-test              & $25.9\%$ & $10.0\%$ & $7.2\%$ \\
same-class             & $0.0\%$  & $3.1\%$  & $5.4\%$ \\
unrelated              & $70.7\%$ & $82.5\%$ & $86.7\%$ \\
unmeasurable           & $3.4\%$  & $4.3\%$  & $0.8\%$ \\
\bottomrule
\end{tabular}
\caption{How wrong the wrong answers are, under symmetric leave-source-out, for the cases the gate answers
against those it defers. The LLM column is Qwen2.5-VL-7B, fifth of our eight small models by tail Recall@1
($1.7\%$), on its wrong top-1 in the ultra-rare stratum.}
\label{tab:errgrade}
\end{table}

Table~\ref{tab:errgrade} gives the result, on the globally answered decile, $200$ of the $2000$ sampled cases,
\emph{not} the ultra-rare-stratified top-$10\%$ ($n{=}478$) of App.~\ref{app:whichcases}. The account predicts the
opposite sign, a large top1--top2 margin means the runner-up is far, so a near-neighbour true diagnosis should
itself have scored high and shrunk the margin, so this contradicted our expectation rather than confirming it.

\textbf{Robustness across grading specifications.} The shift is positive under every
specification we tried, but its size and its significance are not stable: $+15.8$pp as reported, $+14.8$pp if the
tie rule is applied consistently (below), $+14.4$pp with no tie exclusion at all, and $+10.1$pp per-disease. Only
the first two exclude zero. We therefore claim a \emph{direction} and not a resolved effect.

Ties are excluded on both sides. No top-1 is committed where candidates tie at the maximum, so those
cases are excluded from the grading: $281$ where the true diagnosis is among the tied and a further $301$ where
it is not. Both groups sit entirely in the comparison baseline and both have same-test rates about twice it, so
excluding only the first would inflate the shift. The symmetric exclusion is the $+14.8$pp figure above.

\textbf{The same test does not imply the same counselling.} Of the $15$ same-test errors, $4$ share a causative gene and $11$
share only the Orphanet disorder, locus heterogeneity, covered by a panel or exome but not necessarily by a
single-gene test. More consequentially, $7$ of the $15$ pair a purely dominant entity with a purely recessive
one (Robinow, distal renal tubular acidosis). For those patients the assay is right and the recurrence risk, the
relatives selected for cascade testing and the reproductive counselling are all wrong, the three actions our own
Ethics Statement names as the harm channel. The grading axes are blind to inheritance by construction; we added
the check only after the fact, and it removes about half of the comfort the same-test bucket appears to offer.

Exomiser admits the same grading at gene level, relating two genes through the disorders they cause. Its answered
decile is $10.3\%$ same-test, $1.7\%$ same-class, $62.1\%$ unrelated and $25.9\%$ unmeasurable ($n=58$). The last
figure is eightfold the retriever's $3.4\%$, so the two are not on a comparable denominator and we do not read
the comparison as a ranking; renormalised to measurable errors it is $14.0\%$ against the retriever's $26.8\%$.
A phenotype-only run of a genome-aware tool should be expected to do worse here, which is one more reason not to
read our Exomiser numbers as a verdict on Exomiser.

\textbf{Margin confidence and general case difficulty.} Every row above is a difficulty proxy, so a sharper version of
the objection survives: perhaps the margin identifies cases that are easy for \emph{any} predictor, in which case
``retriever-first'' would weaken to ``gate externally, then let either predictor deliver.'' We test it by scoring
the \emph{bare LLM} on exactly the cases the margin answers (Table~\ref{tab:banddiff}, $N{=}2000$ per model). Both
predictors do improve inside the band, so a shared-difficulty component is real, but they do not converge: the
retriever reaches $81.1$--$83.9\%$ against the LLM's $1.1$--$44.0\%$, an advantage of $+39$ to $+83$pp for every
one of the ten configurations, including DeepSeek-V4-Pro with reasoning \emph{on} ($44.0\%$ vs.\ $83.3\%$). For
two of the weaker models the margin is mildly \emph{anti}-correlated with LLM correctness ($1.7\!\to\!1.1\%$,
$1.6\!\to\!1.1\%$). The gate therefore selects cases \emph{the retriever} solves, not cases that are easy in
general, which is exactly the predictor-selection reading.

\begin{table*}[t]\centering\small
\begin{tabular}{@{}lcccc@{}}
\toprule
ultra-rare tail & \multicolumn{2}{c}{bare LLM Recall@1} & \multicolumn{2}{c}{retriever Recall@1} \\
\cmidrule(lr){2-3}\cmidrule(lr){4-5}
 & all & \textbf{margin decile} & all & \textbf{margin decile} \\
\midrule
Qwen2.5-VL-7B    & $1.7$  & $1.1$  & $25.6$ & $83.9$ \\
Qwen2.5-14B      & $4.2$  & $14.0$ & $25.6$ & $83.9$ \\
Qwen2.5-VL-32B   & $4.3$  & $16.7$ & $24.6$ & $81.1$ \\
Llama-3.1-8B     & $4.6$  & $19.4$ & $25.6$ & $83.9$ \\
Mistral-7B-v0.3  & $1.6$  & $1.1$  & $25.6$ & $83.9$ \\
InternLM2.5-7B   & $1.4$  & $6.5$  & $25.6$ & $83.9$ \\
Yi-1.5-9B        & $0.7$  & $2.2$  & $25.6$ & $83.9$ \\
DeepSeek-V4-Flash & $15.8$ & $30.1$ & $25.6$ & $83.9$ \\
DeepSeek-V4-Pro (reason.\ off) & $18.6$ & $34.4$ & $25.6$ & $83.9$ \\
DeepSeek-V4-Pro (reason.\ on)  & $22.4$ & $\mathbf{44.0}$ & $25.4$ & $\mathbf{83.3}$ \\
\bottomrule
\end{tabular}
\caption{Recall@1 (\%) on the ultra-rare tail, overall and restricted to the decile the retriever's own
margin answers. The retriever column varies only with the case set a given model produced parseable output
for; BioMistral is omitted, emitting a well-formed differential in ${\sim}0.4\%$ of cases.}
\label{tab:banddiff}
\end{table*}

\section{Table Notes}

\label{app:tabnotes}
\textbf{Table~\ref{tab:collapse}.} ``less/ultra ratio'' $=$ less-rare over ultra-rare per-case R@1; \emph{per-disease
the LLM collapse is only $\sim\!2$--$6\times$} (App.~\ref{app:macro}). Yi is weak on both bins, so its ratio
does not describe a fall from a working anchor; we report it rather than dropping it. The frontier model edges just above Exomiser's (gene-level, not directly comparable) score.
For the retriever row, a stricter same-publication control brackets the tail at $5.2$--$25.6\%$, where it no
longer clears the frontier model (App.~\ref{app:decon}), and its per-case $1.2\times$ \emph{reverses} to a
per-disease rise (App.~\ref{app:macro}). BioMistral-7B is medically adapted and follows the requested output
format poorly (App.~\ref{app:repro}). Exomiser's gene-level score is not directly comparable to the disease-level rows: it dwarfs the
small-model tail and is comparable to our retriever (App.~\ref{app:exomiser}).

\textbf{Table~\ref{tab:method}.} The ceiling is $\min(1,p/c)$ at the LLM's own base rate, so no confidence
estimator, verbalized, log-probability, semantic entropy or conformal, can close it (\S\ref{sec:calib}). The
LLM's \emph{observed} band is reported at the coverage its confidence can actually resolve: that signal takes
only $5$--$14$ distinct values here, so a $10\%$ band does not exist for it (for Qwen-32B \emph{no} case is
strictly above the cut, all $90$ would come from one $140$-case tie block), whereas the margin has $682$ values
and resolves $10\%$ exactly. The mid-size models rank their own answers well (own-answer AUROC $0.78$/$0.90$;
App.~\ref{app:calib}) and still cannot reach a safe point: the barrier is base rate, not discrimination. The
retriever-first row is the deployable hybrid of \S\ref{sec:method}, so its spread is the LLM's marginal effect on
the \emph{same} cases (${\approx}94$ cases, per-cell $95\%$ CI ${\approx}\pm9$pp; the margin-vs-ceiling separation
far exceeds that; per-disease the band is $64.6\%$, App.~\ref{app:whichcases}). Under the strict leakage control
of App.~\ref{app:decon} the retriever's base falls to $5.2\%$ and this band with it ($12.3\%$). The $25\%$-coverage
point and the dominated agreement gate are in App.~\ref{app:method}. Scored under the tie-breaking most
favourable to the LLM, the margin's band on the small models is $74$--$83\%$ against at most $3$--$28\%$ for
verbalized confidence at matched $10\%$ coverage.

\section{Reproducibility}

\label{app:repro}
\textbf{Compute.} All local runs are on one machine with four NVIDIA RTX $6000$ Ada Generation GPUs ($48$\,GB
each), driver $570.195.03$, CUDA $12.8$. Open-weight inference is single-GPU except the $32$B AWQ model, which is
tensor-parallel across four. The frontier configurations are API calls and consume no local GPU
time, and the decontamination, entity-linking and retrieval analyses are CPU- or single-GPU jobs. We release
the per-run logs rather than a single total, since the runs were made incrementally over the project.

\textbf{Software.} Python $3.11$, PyTorch $2.11$ built against CUDA $12.8$, vLLM for open-weight decoding, and
\texttt{transformers} for SapBERT. All data are public: Phenopacket Store v0.1.27, HPO/HPOA release 2026-06-23, and Orphanet cross-references. We
release the full pipeline (parsing, per-model inference, stratified analysis, the true-path-propagated retriever
with source-decontamination, the rerank probe with its shuffle control, and the retriever-margin triage and
hybrid analysis), the
exact prompts, model revisions, the vLLM version, and per-case outputs including the matched OMIM labels.
Free-text names are linked by nearest-neighbor cosine over the OMIM vocabulary with SapBERT
(\texttt{cambridgeltl/\allowbreak SapBERT-\allowbreak from-\allowbreak PubMedBERT-\allowbreak fulltext}); this linking is sensitive to the transformers/model
version, which we pin, it degenerates under \texttt{transformers}~5.0, so we fix the version under which linking
is stable (the retriever and rerank numbers use no name matching and reproduce exactly). Decoding is greedy
(temperature $0$). The three Qwen models and Llama-3.1-8B emit a well-formed five-item differential in ${>}99\%$
of cases; BioMistral does so in only $\sim\!0.4\%$ (and a parseable overall confidence in $\sim\!1.5\%$), so its
low recall partly reflects instruction-following failure and we do not lean on it. \textbf{Frontier models.}
DeepSeek-V4-Flash and DeepSeek-V4-Pro are queried through the DeepSeek API on the identical $N{=}2000$ sample and
prompt, temperature $0$, with reasoning \emph{disabled} (\texttt{thinking:\{type:disabled\}}) so the comparison to
the greedy open models is fair; both emit a well-formed differential in $100\%$ of cases. The API returns no
usable token log-probabilities, so the log-probability signal (\S\ref{sec:calib}) is reported for the open-weight
models only; all other signals (verbalized confidence, retriever margin/agreement) are computed the same way for
every model. The \emph{reasoning-on} run uses the same endpoint and prompt with \texttt{thinking} enabled; $90\%$
of cases return well-formed output within the token budget (exhaustion on the hardest cases), and reasoning-on
numbers are computed on that subset.

\textbf{Determinism and seeds.} The $N{=}2000$ evaluation sample is drawn once with seed $0$ and reused by
every system, so all rows are scored on the same cases (Table~\ref{tab:collapse} states the two exceptions,
which are parsing failures rather than different samples). Masking in the entity-linking and retrieval arms
uses seed $0$, and the three rates are nested prefixes of one permutation rather than three independent
draws. Retriever figures throughout use \texttt{PYTHONHASHSEED=0}. Changes in floating-point summation order still
move the base rate and the raw-score gate by about $0.5$pp, and the margin gate reads $79.7\%$ or $81.0\%$
according to which decontamination implementation produced the run, a $1.3$pp spread we quote rather than
average. Numbers sensitive to this are the retriever's absolute levels; the margin-versus-raw ordering is
not, since both are computed from the same run.

\textbf{Artifacts and terms.} Every corpus we use is a publicly released research resource: Phenopacket Store,
HPO and HPOA, Orphanet, SciFact, BC5CDR and MedMentions, each used under its own published terms. The model
weights are the vendors' public releases at the revisions we pin. We release code, prompts, the parsing and
linking pipeline, and per-case derived outputs including matched OMIM identifiers; we redistribute no
source case-report text and no patient attributes beyond the structured HPO terms and gold labels the
corpora already publish.

\end{document}